\documentclass{article}

\usepackage[verbose=true,letterpaper]{geometry}

\usepackage{amsmath,amsthm,amssymb,mathtools}

\usepackage{algorithm}

\usepackage[hidelinks, hypertexnames=false, colorlinks=true, citecolor=Navy, linkcolor=Maroon, urlcolor=Orchid, bookmarksnumbered, unicode, backref=page]{hyperref}

\usepackage[noend]{algpseudocode}

\algnewcommand{\algorithmicinput}{\textbf{Input:}}
\algnewcommand{\Input}{\item[\algorithmicinput]}
\algnewcommand{\algorithmicoutput}{\textbf{Input:}}
\algnewcommand{\Output}{\item[\algorithmicoutput]}
\algnewcommand{\Break}{\textbf{break}}

\usepackage{cleveref}
\crefname{equation}{eq.}{eqs.}
\Crefname{equation}{Eq.}{Eqs.}
\crefname{algorithm}{Algorithm}{Algorithms}
\Crefname{algorithm}{Algorithm}{Algorithms}
\crefname{step}{Step}{Steps}
\Crefname{step}{Step}{Steps}
\crefname{theorem}{Theorem}{Theorems}
\Crefname{theorem}{Theorem}{Theorems}
\crefname{lemma}{Lemma}{Lemmas}
\Crefname{lemma}{Lemma}{Lemmas}
\crefname{proposition}{Proposition}{Propositions}
\Crefname{proposition}{Proposition}{Propositions}
\crefname{assumption}{Assumption}{Assumptions}
\Crefname{assumption}{Assumption}{Assumptions}
\crefname{definition}{Definition}{Definitions}
\Crefname{definition}{Definition}{Definitions}
\crefname{figure}{Figure}{Figures}
\Crefname{figure}{Figure}{Figures}
\crefname{table}{Table}{Tables}
\Crefname{table}{Table}{Tables}
\crefname{section}{Section}{Sections}
\Crefname{section}{Section}{Sections}
\crefname{appendix}{Appendix}{Appendices}
\Crefname{appendix}{Appendix}{Appendices}

\usepackage[utf8]{inputenc} 
\usepackage[T1]{fontenc}    
\usepackage{amsfonts}       
\usepackage{nicefrac}       
\usepackage{microtype}      
\usepackage{bm}             
\usepackage{bbm}

\usepackage{graphicx}
\usepackage[svgnames]{xcolor}
\usepackage{subcaption}
\usepackage{wrapfig}

\usepackage{tabularx}       
\usepackage{booktabs}       

\usepackage[compress]{natbib}

\usepackage{thmtools}
\usepackage{thm-restate}
\usepackage{url}            
\usepackage{braket}
\usepackage{mleftright}
\mleftright
\usepackage{multirow}
\usepackage{lipsum}
\allowdisplaybreaks
\usepackage{lscape}
\usepackage{tcolorbox}
\usepackage{xifthen}
\usepackage{xargs}
\usepackage{xspace}
\usepackage{enumerate}
\usepackage[color={red!100!green!33},colorinlistoftodos,prependcaption,textsize=small]{todonotes}

\usepackage{tikz}
\usetikzlibrary{backgrounds}
\usetikzlibrary{arrows}
\usetikzlibrary{shapes,shapes.geometric,shapes.misc}

\tikzstyle{tikzfig}=[baseline=-0.25em,scale=0.5]

\pgfkeys{/tikz/tikzit fill/.initial=0}
\pgfkeys{/tikz/tikzit draw/.initial=0}
\pgfkeys{/tikz/tikzit shape/.initial=0}
\pgfkeys{/tikz/tikzit category/.initial=0}

\pgfdeclarelayer{edgelayer}
\pgfdeclarelayer{nodelayer}
\pgfsetlayers{background,edgelayer,nodelayer,main}

\tikzstyle{none}=[inner sep=0mm]

\tikzstyle{rectangle}=[fill=white, draw=black, shape=rectangle]
\tikzstyle{circle}=[fill=white, draw=black, shape=circle]
\tikzstyle{vertex}=[fill=black, draw=none, shape=circle]
\tikzstyle{textbox}=[fill=white, draw=none, shape=rectangle]

\tikzstyle{line}=[-, fill=none]
\tikzstyle{rightarrow}=[->]
\tikzstyle{leftarrow}=[fill=none, <-]

\newcommand{\Bcal}{\mathcal{B}}

\newcommand{\Dcal}{\mathcal{D}}

\newcommand{\Ncal}{\mathcal{N}}

\newcommand{\Pcal}{\mathcal{P}}

\newcommand{\Etl}{\tilde{E}}

\newcommand{\Vtl}{\tilde{V}}

\newcommand{\phat}{\hat{p}}
\newcommand{\qhat}{\hat{q}}

\newcommand{\gtl}{\tilde{g}}

\newcommand{\wtl}{\tilde{w}}

\newcommand{\R}{\mathbb{R}}
\newcommand{\N}{\mathbb{N}}
\newcommand{\Z}{\mathbb{Z}}

\newcommand{\ones}{\bm{1}}
\newcommand{\zeros}{\bm{0}}
\newcommand{\Ord}{\mathrm{O}}

\newtheorem{theorem}{Theorem}
\newtheorem{lemma}{Lemma}
\newtheorem{proposition}{Proposition}
\newtheorem{assumption}{Assumption}
\theoremstyle{definition}
\newtheorem{remark}{Remark}
\newtheorem{example}{Example}

\DeclareMathOperator*{\minimize}{minimize}
\DeclareMathOperator*{\maximize}{maximize}
\DeclareMathOperator{\subto}{subject\ to}
\DeclareMathOperator*{\argmax}{argmax}
\DeclareMathOperator*{\argmin}{argmin}

\DeclareMathOperator{\sign}{sign}

\DeclareMathOperator{\conv}{conv}
\DeclareMathOperator{\dom}{dom}

\DeclareMathOperator{\E}{\mathbb{E}}

\DeclarePairedDelimiter\abs{\lvert}{\rvert}
\DeclarePairedDelimiter\norm{\lVert}{\rVert}
\DeclarePairedDelimiter\floor{\lfloor}{\rfloor}
\DeclarePairedDelimiter\ceil{\lceil}{\rceil}
\DeclarePairedDelimiter\round{\lfloor}{\rceil}
\DeclarePairedDelimiter\iprod{\langle}{\rangle}
\let\set\relax
\DeclarePairedDelimiter\set{\{}{\}}
\let\Set\relax
\DeclarePairedDelimiterX\Set[2]{\{}{\}}{\mspace{2mu}{#1}\;\delimsize|\;{#2}\mspace{2mu}}
\DeclarePairedDelimiter\brc{[}{]}
\DeclarePairedDelimiterX\Brc[2]{[}{]}{\mspace{2mu}{#1}\;\delimsize|\;{#2}\mspace{2mu}}
\DeclarePairedDelimiter\prn{(}{)}
\DeclarePairedDelimiterX\Prn[2]{(}{)}{\mspace{2mu}{#1}\;\delimsize|\;{#2}\mspace{2mu}}

\newif\iffigure
\figurefalse

\usepackage{nicematrix}
\usepackage{svg}
\usepackage{fullpage}
\usepackage{libertine}

\usepackage{dsfont}

\title{Rethinking Warm-Starts with Predictions: \\Learning Predictions Close to Sets of Optimal Solutions for \\Faster $\text{L}$-/$\text{L}^\natural$-Convex Function Minimization}

\author{%
Shinsaku Sakaue\\
The University of Tokyo\\
Tokyo, Japan\\
\href{mailto:sakaue@mist.i.u-tokyo.ac.jp}{sakaue@mist.i.u-tokyo.ac.jp}
\and
Taihei Oki\\
The University of Tokyo\\
Tokyo, Japan\\
\href{mailto:oki@mist.i.u-tokyo.ac.jp}{oki@mist.i.u-tokyo.ac.jp}
}

\date{}

\begin{document}

\maketitle

\begin{abstract}
	An emerging line of work has shown that machine-learned predictions are useful to warm-start algorithms for discrete optimization problems, such as bipartite matching. Previous studies have shown time complexity bounds proportional to some distance between a prediction and an optimal solution, which we can approximately minimize by learning predictions from past optimal solutions. However, such guarantees may not be meaningful when multiple optimal solutions exist. Indeed, the dual problem of bipartite matching and, more generally, \emph{$\text{L}$-/$\text{L}^\natural$-convex function minimization} have \emph{arbitrarily many} optimal solutions, making such prediction-dependent bounds arbitrarily large. To resolve this theoretically critical issue, we present a new warm-start-with-prediction framework for $\text{L}$-/$\text{L}^\natural$-convex function minimization. Our framework offers time complexity bounds proportional to the distance between a prediction and the \emph{set of all optimal solutions}. The main technical difficulty lies in learning predictions that are provably close to sets of all optimal solutions, for which we present an online-gradient-descent-based method. We thus give the first polynomial-time learnability of predictions that can provably warm-start algorithms regardless of multiple optimal solutions.
\end{abstract}

\newcommand{\Ln}{\texorpdfstring{L${}^\natural$\xspace}{L-natural}}
\newcommand{\Tprj}{{\mbox{${T}_\mathrm{prj}$}}}
\newcommand{\Tloc}{{\mbox{${T}_\mathrm{loc}$}}}
\newcommand{\Tineq}{{\mbox{${T}_\mathrm{ineq}$}}}

\newcommand{\ind}[1]{\mathds{1}_{#1}}

\newcommand{\pcirc}{p^\circ}
\newcommand{\Mbf}{\mathbf{M}}
\newcommand{\restr}[2]{#1 \mathbin{|} #2}
\newcommand{\contract}[2]{#1 \mathbin{/} #2}
\newcommand{\ycirc}{y^\circ}

\newcommand{\mubar}{\bar{\mu}}
\newcommand{\stpath}{P}

\section{Introduction}\label{sec:introduction}
Algorithms with predictions \citep{Mitzenmacher2021-bq}---improving algorithms performance with predictions learned from data---is a rapidly growing research field.
Seminal work by \citet{Dinitz2021-sd} has initiated the study of using predictions to warm-start discrete optimization algorithms.
A brief description of their result for weighted perfect bipartite matching problems is as follows.
Consider finding a maximum-weight perfect matching in a bipartite graph $(V, E)$ with an equal-sized bipartition $V = L\cup R$ and edge weights $w\in\Z^E$.\footnote{While the minimum-weight setting was originally studied, we describe the maximum-weight setting as in \citep{Sakaue2022-jr}.}
The dual of this problem is written as a linear program (LP) with variables $p = (s, t) \in \R^V$:
\begin{align}\label{prob:matching-dual}
  \begin{aligned}
	& \minimize_{s \in \R^L, t \in \R^R} \quad
	\sum_{i \in L} s_i - \sum_{j \in R} t_j \\
	& \subto \quad
	s_i - t_j \ge w_{ij} \quad
	\forall
	ij \in E.
  \end{aligned}
\end{align}
The authors showed that given a prediction $\phat \in \R^V$ of \emph{some} optimal dual solution $p^*$, we can efficiently convert $\phat$ into an initial feasible solution $\pcirc$, and the Hungarian method warm-started by $\pcirc$ runs in $\Ord(|E|\sqrt{|V|}\norm{p^* - \phat}_1)$ time, whereas the worst-case running time is $\Ord(|E||V|)$.
Moreover, given about $\Omega(|V|^3)$ optimal solutions $p^*$ drawn i.i.d.\ from a fixed distribution $\Dcal$, we can learn $\phat$ that approximately minimizes $\E_{p^*\sim \Dcal}\brc{\norm{p^* - \phat}_1}$ via empirical risk minimization.
In a nutshell, learning predictions from past optimal solutions can provably accelerate the Hungarian method.

The above argument, however, has a subtle but critical pitfall: there \emph{always} exist \emph{arbitrarily many} optimal dual solutions.
To see this, let $p^\prime$ be an optimal solution to~\eqref{prob:matching-dual}.
Then, adding any vector in the all-one direction to $p^\prime$ does not change the objective function value and the left-hand sides of the constraints; therefore, $p^\prime + \zeta \ones$ is also optimal for all $\zeta \in \R$.
Hence, the $\Ord(|E|\sqrt{|V|}\norm{p^* - \phat}_1)$-time bound requires us to select one optimal solution $p^*$, and the bound can be arbitrarily large if selected $p^*$ is far away from $\phat$.
One may think such a concern is unnecessary since the Hungarian method warm-started by $\phat$ would return $p^*$ close to $\phat$.
This idea, however, makes $p^*$ selected depending on $\phat$, and no existing results on the learnability of predictions $\phat$ can deal with such dependence.
We further detail this issue in \cref{asec:remark}.

The cause of this troublesome situation is that an optimal solution $p^*$ is not unique for a given bipartite matching instance.
By contrast, \emph{the set of all optimal solutions} is unique.
Therefore, the distance between $\phat$ and the set of all optimal solutions (or equivalently, the minimum distance between $\phat$ and an optimal solution) is a well-defined measure to quantify the speed-up gained by using prediction $\phat$.
Moreover, this idea can strengthen distance-dependent time complexity bounds by taking the minimum among all optimal solutions.

\begin{table}[tb]
	\caption{Improved time complexity bounds for the problems studied in \citep{Sakaue2022-jr} (see also Table 1 therein).
  For weighted perfect bipartite matching and discrete energy minimization, $n$ and $m$ are the sizes of vertex and edge sets, respectively.
  For weighted matroid intersection, $r$ is the rank of matroids, $n$ is the ground-set size, and $\tau$ is the running time of independence oracles. }\label{table:results}
	\vskip 0.15in
	\centering
	\begin{tabular}{cc} \toprule
	Problem  & Time complexity  \\ \midrule
	Weighted perfect bipartite matching  & $\Ord(m\sqrt{n} \cdot \mubar(\phat; g))$  \\
	Weighted matroid intersection  & $\Ord(\tau n r^{1.5} \cdot \mubar(\phat; g))$  \\
	Discrete energy minimization & $\Ord(mn^2 \cdot \mubar(\phat; g))$  \\
	\bottomrule
	\end{tabular}
  \vskip -0.1in
\end{table}

\subsection{Our Contribution}\label{subsec:our-result}
We present a new framework with time complexity bounds proportional to the distance between prediction $\phat$ and the set of all optimal solutions.
Building on a recent improvement \citep{Sakaue2022-jr} of \citep{Dinitz2021-sd}, we develop our framework for \emph{L-/\Ln-convex function minimization}, a broad class of discrete optimization problems, such as the weighted perfect bipartite matching, weighted matroid intersection, and discrete energy minimization.
The pitfall mentioned above also exists in L-/\Ln-convex minimization (see \cref{rem:multiple-opt}) and has remained open in the prior work.

We here give some informal definitions for convenience (see \cref{section:preliminaries} for details).
Let $g$ be an L-/\Ln-convex function to be minimized, which represents both an objective function and constraints, taking $+\infty$ if infeasible.
We quantify the distance between prediction $\phat$ and the set, $\conv(\argmin g)$, of all optimal solutions with the \emph{$\ell^\pm_\infty$-norm}; let $\mubar(\phat; g)$ denote this distance.
Our high-level idea is to use $\mubar(\phat; g)$ instead of any distance defined with some fixed optimal $p^*$.
Although the idea is simple, it involves two unprecedented challenges:
(i) to show that algorithms warm-started with $\phat$ run in time proportional to $\mubar(\phat; g)$ and
(ii) to learn $\phat$ that approximately minimizes $\E_g[\mubar(\phat; g)]$.
We describe how to achieve them.

\Cref{sec:main-framework} shows that an L-/\Ln-convex minimization method warm-started with $\phat$ enjoys a time complexity bound proportional to $\mubar(\phat; g)$.
As with \citep{Sakaue2022-jr}, we employ the steepest descent method for solving L-/\Ln-convex minimization, and we additionally utilize a fact that it converges to an optimal solution closest to an initial feasible solution.
Our analysis applies to all the problems studied in \citep{Sakaue2022-jr} and improves their time complexity bounds, which are proportional to the $\ell_\infty$-distance, $\norm{p^* - \phat}_\infty$, between $\phat$ and some fixed optimal $p^*$.
\cref{table:results} summarizes our improved time complexity bounds, and \Cref{fig:mu_is_smaller} illustrates how our idea improves their previous bounds.

\begin{figure}
	\centering
  \includegraphics[width=.48\linewidth]{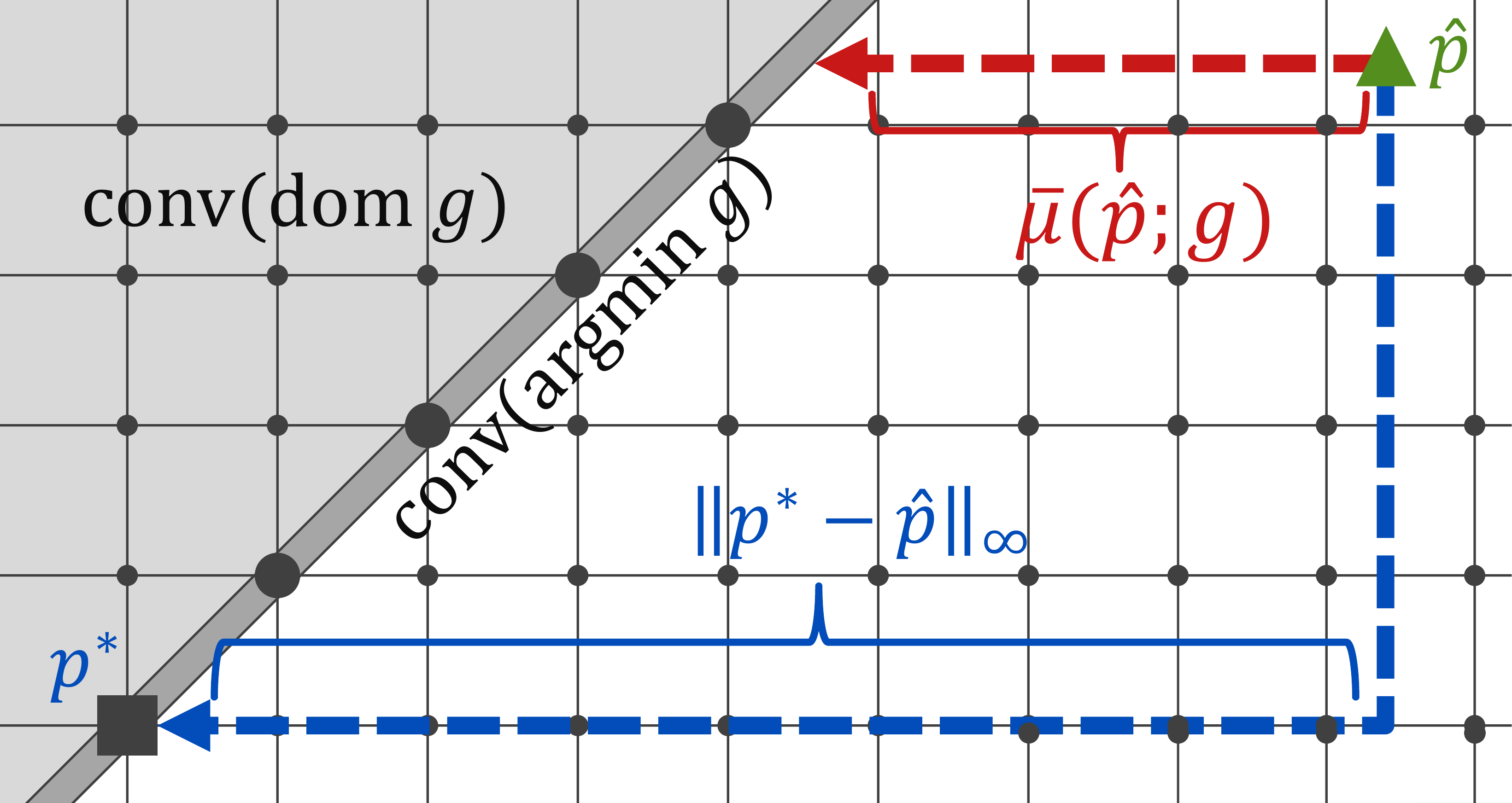}
	\caption{
		Comparison of our result with $\mubar(\phat; g)$ and the previous result with $\norm{p^* - \phat}_\infty$ \citep{Sakaue2022-jr}.
    Imagine L-convex minimization in $\Z^2$.
    Let $\phat \in \R^2$ be a given prediction, $g$ an L-convex function, and $p^*$ an optimal solution selected for $g$.
    The gray area, $\conv(\dom g)$, is (the convex hull of) the feasible region, and the darker area, $\conv(\argmin g)$, is the set of all optimal solutions.
    As~discussed in \cref{sec:introduction}, $\conv(\argmin g)$ has freedom in the all-one direction;
    hence, $\norm{p^* - \phat}_\infty$ can be arbitrarily large if $p^*$ is far from $\phat$.
    By contrast, $\mubar(\phat; g)$ uniquely represents the minimum $\ell^\pm_\infty$-distance between $\phat$ and $p^* \in \conv(\argmin g)$ closest to $\phat$.
	}
  \label{fig:mu_is_smaller}
\end{figure}

\Cref{sec:learning} presents how to learn $\phat$ that approximately minimizes $\E_g[\mubar(\phat; g)]$.
Similar to \citep{Khodak2022-sf,Sakaue2022-jr}, we prove a regret bound of the online gradient descent method (OGD) for learning $\phat$ and obtain a sample complexity bound via online-to-batch conversion.
Our contribution is to obtain those bounds for $\mubar(\phat; g)$, not for any distance between $\phat$ and fixed optimal $p^*$.
The main difficulty lies in computing subgradients of $\mubar(\cdot; g)$ used in OGD, for which we use a connection between $\mubar(\cdot; g)$ and a shortest path problem and Danskin's theorem (see \cref{subsec:subgradient}).
Also, computing a subgradient requires an inequality system that represents the set, $\conv(\argmin g)$, of all optimal solutions, for which we give polynomial-time methods (see \cref{sec:minimizer-set}).
We thus obtain the first polynomial-time learnability of predictions that can provably warm-start algorithms regardless of multiple optimal solutions.
Furthermore, our regret bound is tight up to constant factors, as shown in \cref{asec:lower-bound}.

\paragraph{Remarks and Limitations.}
Since we focus on theoretically refining prediction-dependent time complexity bounds, the practical impact would be somewhat limited; still, \cref{asec:experiment} presents some promising empirical results.
Also, we do not discuss worst-case bounds since we can bound the worst-case runtime by executing standard algorithms with worst-case guarantees in parallel, as in \citep[Section~6]{Sakaue2022-jr}.
We emphasize that our motivation is to warm-start simple algorithms with predictions, while theoretically fast algorithms, which are often hard to implement and slow in practice, sometimes enjoy better time complexity bounds.

\subsection{Related Work}\label{subsec:related-work}
A flurry of recent work has been devoted to going beyond the worst-case analysis of algorithms using predictions.
While improving competitive ratios of online algorithms has occupied a central place \citep{Purohit2018-yx,Bamas2020-rj,Lykouris2021-xn,Azar2022-rh}, the idea has gained increasing attention in various areas, including algorithmic game theory \citep{Agrawal2022-cb} and data structures \citep{Boffa2022-sl}.
A comprehensive list of papers in this field is provided by \citet{Lindermayr_undated-mr}.

Besides \citep{Sakaue2022-jr}, \citet{Chen2022-li} have improved the prediction-dependent time complexity bound of \citep{Dinitz2021-sd} and presented general results for warm-starting various graph algorithms and learning predictions.
\citet{Polak2022-je} have used predictions to warm-start a maximum-flow algorithm.
Although those existing studies have provided time complexity bounds depending on some distance, $\norm{p^* -\phat}$, between an optimal solution $p^*$ and a prediction $\phat$, the non-uniqueness of $p^*$, despite its prevalence, has not been well discussed---$p^*$ has been (implicitly) assumed to be unique.\footnote{In \citep{Polak2022-je}, their bound is said to hold for every optimal solution, but its non-uniqueness is not correctly handled when learning predictions. See \cref{asubsec:polak} for details.}
Researchers have also used predictions to accelerate algorithms for support estimation \citep{Eden2021-wa}, the shortest path problem \citep{Feijen2021-mi}, generalized sorting \citep{Lu2021-cn}, nearest neighbor search \citep{Andoni2022-iy}, and clustering \citep{Ergun2022-gq}, while their time complexity analyses are different from those of warm-starts with predictions.

The L-/\Ln-convexity is a fundamental notion in \emph{discrete convex analysis} \citep{Murota2003-bq}, a discrete analog of convex analysis.
It enables us to see various discrete optimization algorithms as the steepest descent method.
This viewpoint offers a geometric understanding of warm-starts with predictions; that is, a prediction closer to an optimum is naturally better since an algorithm iteratively approaches an optimum.
Although \citep{Sakaue2022-jr} is also based on the steepest descent method, they did not utilize a notable property that it converges to an optimal solution closest to an initial point (see \cref{proposition:convergence-iterations}), which is a key to obtaining our result.


\section{Preliminaries}\label{section:preliminaries}
Let $\floor{\cdot}$, $\ceil{\cdot}$, and $\round{\cdot}$ be the element-wise ceiling, floor, and rounding, respectively, where $\round*{\cdot}$ rounds down $0.5$ fractional parts.
For $n \in \N$, let $V = \set*{1,2,\dots,n}$ be a finite ground set of size $n$.
Let $\zeros, \ones \in \R^V$ denote the all-zero and all-one vectors, respectively.
For $S\subseteq \R^V$, let $\conv(S) \subseteq \R^V$ be the convex hull of $S$ and $\delta_S$ the indicator function of $S$, i.e., $\delta_S(p)=0$ if $p \in S$ and $+\infty$ otherwise.

For a function $g:\Z^V \to \R\cup\set{+\infty}$, we define its \textit{effective domain} by $\dom g = \Set*{p \in \Z^V}{g(p) < +\infty}$, which represents the feasible region of a minimization problem of the form $\min_{p\in\Z^V} g(p)$.
We say $g$ is \textit{proper} if $\dom g \neq \emptyset$.

\subsection{L-/\Ln-Convex Functions and Sets}\label{subsec:l-convexity}
We overview the properties of L-/\Ln-convex functions and sets.
We refer the reader to \citep{Murota2003-bq} for more details.

Let $g:\Z^V \to \R\cup\set{+\infty}$ be a proper function.
We say $g$ is \textit{L-convex} if $g(p) + g(q) \ge g(p \vee q) + g(p \wedge q)$ for all $p, q \in \Z^V$, where $\vee$ ($\wedge$) is the element-wise maximum (minimum), and there exists $r \in \R$ such that $g(p + \ones) = g(p) + r$ for all $p \in \Z^V$.
Also, $g$ is \textit{\Ln-convex} if $\gtl(p_0, p) \coloneqq g(p - p_0\ones)$ is L-convex on $\Z\times\Z^V$; this is equivalent to $g(p) + g(q) \ge g\prn*{\ceil*{\frac{p+q}{2}}} + g\prn*{\floor*{\frac{p+q}{2}}}$ for all $p, q \in \Z^V$, analogous to the standard convexity of functions on $\R^V$.
Since L-convexity on $\Z\times\Z^V$ and \Ln-convexity on $\Z^V$ are equivalent, we may use whichever is convenient.
If $g_1$ and $g_2$ are L-/\Ln-convex on $\Z^V$ and $g_1 + g_1$ is proper, $g_1 + g_1$ is L-/\Ln-convex.

The L-/\Ln-convex function minimization, $\min_{p \in \Z^V} g(p)$, is known to contain a wide variety of problems, e.g., the dual of weighted bipartite matching, dual of weighted matroid intersection, and discrete energy minimization, which generalizes minimum-cost flow (see \citep[Chapter~9]{Murota2003-bq}).
In what follows, we make the following basic assumption.

\begin{assumption}\label{assump:non-empty-argmin}
  L-/\Ln-convex functions $g$ always have at least one minimizer, i.e., $\argmin g \neq \emptyset$.
\end{assumption}

A non-empty set $S \subseteq \Z^V$ is \textit{L-/\Ln-convex} if its indicator function $\delta_S: \Z^V \to \set*{0, +\infty}$ is L-/\Ln-convex.
Conversely, if $g:\Z^V \to \R\cup\set*{+\infty}$ is L-/\Ln-convex, $\dom g \subseteq \Z^V$ is an L-/\Ln-convex set; furthermore, the set of all minimizers, $\argmin g \subseteq \Z^V$, is also L-/\Ln-convex \citep[Theorem 7.17]{Murota2003-bq}.
We use this fact in \cref{subsec:subgradient}.
L-/\Ln-convex sets enjoy useful inequality-system representations as follows.
\begin{proposition}[{\citet[Sections~5.3 and~5.5]{Murota2003-bq}}]\label{proposition:l-convex-set}
  For a non-empty set $S \subseteq \Z^V$, the following two are equivalent:
  (i) $S$ is an \Ln-convex set and (ii) $S$ is written as
  \begin{align}\label{eq:inequality-system-z}
    \Set*{p \in \Z^V}{\begin{alignedat}{2}
      &\text{$\alpha_i \le p_i \le \beta_i$ for $i \in V$}, \\
      &\text{$p_j - p_i \le \gamma_{ij}$ for distinct $i, j \in V$}
    \end{alignedat}}
  \end{align}
  with some $\alpha_{i} \in \Z\cup\set*{-\infty}$, $\beta_{i} \in \Z\cup\set*{+\infty}$, and $\gamma_{ij} \in \Z\cup\set*{+\infty}$.
  Also, $S$ is L-convex if and only if $S$ is written as in \eqref{eq:inequality-system-z} without box constraints, i.e., $\alpha_i = -\infty$ and $\beta_i = +\infty$.
  The convex hull, $\conv(S) \subseteq \R^V$, of an L-/\Ln-convex set $S$ is also characterized as above replacing $\Z^V$ in \eqref{eq:inequality-system-z} with $\R^V$.
\end{proposition}

\begin{example}
  The dual LP \eqref{prob:matching-dual} of weighted bipartite matching with variables $(s, t) \in \R^{V}$ has constraints $s_i - t_j \ge w_{ij}$ for $ij \in E$.
  These are of the form \eqref{eq:inequality-system-z} and represent the convex hull of the L-convex feasible region, or $\conv(\dom g)$.
  Furthermore, given a maximum weight matching $M^* \subseteq E$, a dual feasible $(s, t)$ is optimal if and only if $s_i - t_j = w_{ij}$ for $ij \in M^*$ (see \cref{thm:matching-complementary}).
  Thus, we can represent the L-convex set of optimal dual solutions, or $\conv(\argmin g)$, with additional inequalities $s_i - t_j \le w_{ij}$ for $ij \in M^*$.
\end{example}


\begin{remark}\label[remark]{rem:multiple-opt}
  The fact that $\argmin g$ is written as in \eqref{eq:inequality-system-z} immediately implies the non-uniqueness of optimal solutions.
  Specifically, if $g$ is L-convex, shifting $p^* \in \argmin g$ in the all-one direction never goes out of $\argmin g$, as discussed in \cref{sec:introduction}, and similar reasoning applies to the \Ln-convex case if such a shifting does not go out of box constraints.
\end{remark}


\subsection{Steepest Descent for L-/\Ln-Convex Minimization}\label{subsection:preliminaries-steepest-descent}

\begin{algorithm}[tb]
	\caption{Steepest descent method}\label{alg:steepest-descent}
	\begin{algorithmic}[1]
		\State $p \gets \pcirc$ {\label[step]{step:init-sol}}
		\Comment{$\pcirc \in \dom g$ is an initial feasible solution.}
		\While{not converged}
		\State $d \gets \argmin \Set*{g(p + d')}{d' \in \Ncal } $ {\label[step]{step:local-opt}}
		\If{$g^\prime(p; d) = 0$} {\label[step]{step:termination-condition}}
		\State \Return $p$
		\EndIf
		\State $\lambda \gets \sup\Set*{\lambda^\prime \in \Z_{>0}}{ g^\prime(p; \lambda^\prime d) = \lambda^\prime g^\prime(p; d)}$
		{\label[step]{step:step-length}}
		\State $p \gets p + \lambda d$ {\label[step]{step:update}}
		\EndWhile
	\end{algorithmic}
\end{algorithm}

We can solve L-/\Ln-convex minimization, $\min_{p \in \Z^V} g(p)$, by using the steepest descent method in \cref{alg:steepest-descent}, which iterates to proceed along a locally steepest descent direction.
The set, $\Ncal$, of local directions is defined as
$\Ncal \coloneqq \set*{0, +1}^V$ if $g$ is L-convex and
$\Ncal \coloneqq \set*{0, -1}^V \cup \set*{0, +1}^V$ if $g$ is \Ln-convex.
Let $g^\prime(p; d) \coloneqq g(p + d) - g(p)$ denote the \textit{slope} of $g$ at $p \in \dom g$ in the direction of $d \in \Z^V$.

We detail how \cref{alg:steepest-descent} works.
Starting from an initial point $\pcirc \in \dom g$, it iteratively performs the following steps:
find a steepest direction $d$ by solving a local optimization problem (\cref{step:local-opt}),
compute a step length $\lambda \ge 1$ (\cref{step:step-length}),\footnote{\Cref{step:step-length} computes a so-called \textit{long-step} $\lambda$ \citep{Fujishige2015-qo,Shioura2017-gf}. If computing $\lambda$ is costly, we can instead set $\lambda \gets 1$ without affecting \cref{proposition:convergence-iterations} and the subsequent analysis.} and update a current solution $p$ by adding $\lambda d$ to $p$ (\cref{step:update}).
If slope $g^\prime(p; d)$ in some steepest direction $d$ is zero (\cref{step:termination-condition}), $p$ is ensured to be optimal (due to the L-/\Ln-convexity of $g$).
In short, \cref{alg:steepest-descent} minimizes an L-/\Ln-convex function $g$ by iteratively solving local optimization problems in \cref{step:local-opt}.

A remarkable property of \cref{alg:steepest-descent} is that the number of iterations is bounded by the distance between an initial point $\pcirc$ and an optimal solution $p^* \in \argmin g$ closest to $\pcirc$.
We introduce some definitions to describe this property more precisely.
For any $p \in \R^V$, we define the \emph{$\ell_\infty^\pm$-norm} as
\begin{align*}
	\norm{p}_\infty^\pm = \max_{i\in V}\max\set*{0,+p_i} + \max_{i\in V}\max\set*{0,-p_i},
\end{align*}
which satisfies the axioms of norms.
For any L-/\Ln-convex $g:\Z^V \to \R\cup\set{+\infty}$, we define $\mu(\cdot; g) : \Z^V \to \Z_{\ge 0}$ as a function that returns the $\ell_\infty^\pm$-distance between input $p \in \Z^V$ and an optimal solution $p^* \in \argmin g$ closest to $p$, i.e.,
\begin{align*}
  \mu(p; g) \coloneqq \min\Set*{\norm{p^* - p}_\infty^\pm}{p^*\in \argmin g}.
\end{align*}
Then, \cref{alg:steepest-descent} converges to an optimal solution closest to $\pcirc$ in $\mu(\pcirc; g)+1$ iterations, as in the next proposition.
\begin{proposition}[{\citep[Theorem 1.2]{Murota2014-bv} and \citep[Theorem 6.2]{Fujishige2015-qo}}]\label{proposition:convergence-iterations}
	\Cref{alg:steepest-descent} returns an optimal solution $p^* \in \argmin g$ such that $\norm{p^* - \pcirc}^\pm_\infty = \mu(\pcirc; g)$ in at most $\mu(\pcirc; g) + 1$ iterations.
\end{proposition}

\begin{example}\label{example:matching-is-l-convex}
  We again consider the dual LP \eqref{prob:matching-dual} of weighted perfect bipartite matching.
  Since $w_{ij}$ are integers, we can restrict the domain to $\Z^V$ and reduce the LP to the minimization of an L-convex function $g$, which is a sum of a linear objective function and the indicator function of the L-convex feasible region.
  As in \citep[Section~3.1]{Sakaue2022-jr}, we can reduce local optimization in \cref{step:local-opt} to a maximum cardinality matching problem.
  If we solve it with the $\Ord(m\sqrt{n})$-time Hopcroft--Karp algorithm, 
  \cref{alg:steepest-descent} runs in $\Ord(m\sqrt{n} \cdot \mu(\pcirc; g))$ time, which can be faster than the $\Ord(mn)$-time Hungarian method if $\mu(\pcirc; g)$ is small.
  Indeed, \cref{alg:steepest-descent} with a fixed feasible $\pcirc$ closely resembles the Hungarian method (see \citep[Section~18.5b]{Schrijver2003-ol}).
\end{example}

For later use, we also define $\mubar(\cdot; g): \R^V \to \R_{\ge0}$ as
\begin{align*}
  \mubar(p; g) \coloneqq \min\Set*{\norm{p^* - p}_\infty^\pm}{p^*\in\conv(\argmin g)},
\end{align*}
which is a continuous extension of $\mu(\cdot; g)$ and is helpful in benefiting from real-valued predictions.
Note that using $\mubar$ instead of $\mu$ only strengthens time complexity bounds since $\mubar(p; g) \le \mu(p; g)$ for all $p \in \Z^V$.
Indeed, $\mubar(p; g) = \mu(p; g)$ holds for all $p \in \Z^V$, which we can prove by confirming the existence of integral $p^* \in \conv(\argmin g)$ that attains the minimum $\ell_\infty^\pm$-distance.
See \cref{asec:proof-of-lem-musimubar} for the proof.

\begin{restatable}{lemma}{muismubar}\label{lem:muismubar}
  Let $g:\Z^V\to\R\cup\set*{+\infty}$ be an L-/\Ln-convex function.
  For every $p \in \Z^V$, it holds that $\mu(p; g) = \mubar(p; g)$.
\end{restatable}

\section{Time Complexity Bound}\label{sec:main-framework}
We give an improved prediction-dependent time complexity bound for L-/\Ln-convex minimization.
As with \citep{Dinitz2021-sd,Sakaue2022-jr}, we decompose our framework into three phases:
(i) converting a prediction $\phat \in \R^V$ into an initial feasible solution $\pcirc \in \Z^V$,
(ii) solving a problem with an algorithm warm-started by $\pcirc$,
and
(iii) learning predictions $\phat$.
The following theorem gives formal guarantees to phases (i) and (ii), and \cref{sec:learning} studies phase (iii).

\begin{theorem}\label{theorem:dca-framework}
	Let $g:\Z^V \to \R\cup\set*{+\infty}$ be an L-/\Ln-convex function and $\phat \in \R^V$ a possibly infeasible prediction.
  \begin{description}
		\item[(i)] If we can compute an $\ell^\pm_\infty$-projection $\qhat$ of the prediction $\phat$ onto $\conv(\dom g)$, defined by
		\begin{equation}\label{eq:ellpm-prj}
      \qhat \in \argmin \Set*{ \norm{q - \phat}^\pm_\infty }{ q \in \conv(\dom g) },
    \end{equation}
    in $\Tprj$ time, we can obtain an initial feasible solution $\pcirc = \round*{\qhat} \in \dom g$ in $\Ord(\Tprj + |V|)$ time.
		\item[(ii)] If we can solve local optimization in \cref{step:local-opt} in $\Tloc$ time, \cref{alg:steepest-descent} starting from $\pcirc = \round*{\qhat}$ finds an optimal solution to $\min_{p \in \Z^V} g(p)$ in $\Ord(\Tloc \cdot \mubar(\phat; g))$ time.
	\end{description}
\end{theorem}

\begin{proof}
  The claim of (i) is identical to that of \citep[Theorem 1]{Sakaue2022-jr}, and so is its proof.
  We below prove the claim of (ii) by modifying their original proof.

  Since \cref{proposition:convergence-iterations} says that \cref{alg:steepest-descent} finds an optimal solution in $\Ord(\Tloc\cdot\mu(\pcirc; g))$ time, the second claim holds if $\mu(\pcirc; g) = \Ord(\mubar(\phat; g))$, which can be proved as follows.

  For the given prediction $\phat$, take any $p^* \in \conv(\argmin g)$ that attains $\norm{p^* - \phat}_\infty^\pm = \mubar(\phat; g)$.
  Note that we have
  \begin{align*}
    \mu(\pcirc; g) = \mubar(\pcirc; g) \le \norm{p^* - \pcirc}_\infty^\pm,
  \end{align*}
  where the equality is due to \cref{lem:muismubar} with $\pcirc = \round*{\qhat} \in \Z^V$ and the inequality comes from the definition of $\mubar(\cdot; g)$ and $p^* \in \conv(\argmin g)$.
  From $\pcirc = \round*{\qhat}$ and the fact that rounding changes each entry up to $\pm1/2$, we have
  \[
    \norm{p^* - \pcirc}_\infty^\pm \le \norm{p^* - \qhat}_\infty^\pm +1.
  \]
  Furthermore, the triangle inequality implies
  \[
    \norm{p^* - \qhat}_\infty^\pm \le \norm{p^* - \phat}_\infty^\pm + \norm{\qhat - \phat}_\infty^\pm.
  \]
  Here, we have $\norm{p^* - \phat}_\infty^\pm = \mubar(\phat; g)$ due to the choice of $p^*$.
  Also, $\norm{\qhat - \phat}_\infty^\pm \le \norm{p^* - \phat}_\infty^\pm = \mubar(\phat; g)$ holds since $\qhat$ is defined as in \eqref{eq:ellpm-prj} and $p^* \in \conv(\argmin g) \subseteq \conv(\dom g)$.
  Thus, we obtain $\mu(\pcirc; g) \le 2\mubar(\phat; g) + 1 = \Ord(\mubar(\phat; g))$.
\end{proof}

\Cref{theorem:dca-framework} says that, given a prediction $\phat \in \R^V$, we can solve $\min_{p \in \Z^V} g(p)$ in $\Ord(\Tprj + |V| + \Tloc \cdot \mubar(\phat; g))$ time.
Furthermore, it holds that $\Tprj + |V| \le \Tloc$ in most cases, including all the problems listed in \cref{table:results} (see \citep[Section 3]{Sakaue2022-jr}).
In such cases, our time complexity bound reduces to $\Ord(\Tloc \cdot \mubar(\phat; g))$, and we can obtain the results in \cref{table:results} by substituting the running time of local optimization solvers into $\Tloc$.
For example, in the bipartite-matching case, we can solve local optimization (maximum cardinality matching) with the Hopcroft--Karp algorithm in $\Tloc = \Ord(m\sqrt{n})$ time, thus obtaining the $\Ord(m\sqrt{n}\cdot\mubar(\phat; g))$-time bound in \cref{table:results}.
For $\Tloc$ of the other problems, see \citep[Sections 3.2 and 3.3]{Sakaue2022-jr}.
Note that our bounds in \cref{table:results} are at least as good as those of \citep{Sakaue2022-jr} up to constant factors since we have $\mubar(\phat; g) \le \norm{p^* - \phat}^\pm_\infty \le 2\norm{p^* - \phat}_\infty$ for any $p^*\in \conv(\argmin g)$.

\section{Learning Predictions}\label{sec:learning}
We now discuss how to learn predictions $\phat \in \R^V$.
Following \citep{Khodak2022-sf,Sakaue2022-jr}, we mainly study the online learning setting, where L-/\Ln-convex functions $g_t$ for $t = 1,\dots,T$ are chosen adversarially.
We apply the online gradient descent method (OGD) to online minimization of $\mubar(\cdot; g_t)$ and prove its regret bound.
We then obtain a sample complexity bound via online-to-batch conversion.

The main goal of this section is to prove that OGD enjoys a regret upper bound and runs in polynomial time as follows.

\begin{restatable}{theorem}{learning}\label{theorem:learning}
  Let $C > 0$.
  For an arbitrary sequence of L-/\Ln-convex functions, $g_1,\dots,g_T$, from $\Z^V$ to $\R\cup\set*{+\infty}$, OGD computes predictions $\phat_1,\dots,\phat_T$ that satisfy
  \[
    \sum_{t = 1}^T \mubar(\phat_t; g_t) \le \min_{\phat^* \in {[-C, +C]}^V} \sum_{t = 1}^T \mubar(\phat^*; g_t) + C\sqrt{2nT}.
  \]
  In each round $t$, if an inequality system of $\conv(\argmin g_t)$ can be obtained in $\Tineq$ time (as in \cref{assump:inequality-system}) and $p^*_t \in \argmin g_t$ is given, OGD takes $\Tineq + \Ord(n^2)$ time.
\end{restatable}

Note that the regret bound in terms of $\mubar(\phat; g_t)$ is the main difference from the previous studies,
which consider simpler functions of the form $\norm{p^*_t - \phat}$ with some fixed optimal $p^*_t$.
Our regret bound is as small as that of \citep{Sakaue2022-jr} even though we consider more involved functions, $\mubar(\cdot; g_t)$, and is indeed asymptotically tight as shown in \cref{asec:lower-bound}.

We show in \Cref{subsec:general-ineq-method} that $\Tineq$ is polynomial even when we only have black-box access to $g_t$.
Moreover, \cref{subsec:efficient-ineq-method} shows that $\Tineq$ can be much smaller for the specific problems listed in \cref{table:results}.
The assumption that $p^*_t \in \argmin g_t$ is available usually holds since we learn $\phat_t$ after solving the $t$th instance, $\min_{p\in \Z^V} g_t(p)$.
(If not, we may solve the $t$th instance with standard polynomial algorithms; then OGD runs in polynomial time.)
In the bipartite-matching case, under those assumptions, OGD will turn out to take only $\Ord(m + n \log n)$ time per round (see \cref{subsec:efficient-ineq-method}), which is even faster than a single local optimization step in \cref{alg:steepest-descent}, or the $\Ord(m\sqrt{n})$-time Hopcroft--Karp algorithm.
Therefore, although our learning method is generally slower than the previous ones, it is usually not a serious drawback.


Given \cref{theorem:learning}, we can obtain a sample complexity bound via online-to-batch conversion. 
The proof is almost identical to those of \citep{Khodak2022-sf,Sakaue2022-jr} and thus deferred to \cref{asec:proof-of-cor-sample}.

\begin{restatable}{corollary}{sample}\label{cor:sample}
  Let $\Dcal$ be an (unknown) distribution over L-/\Ln-convex functions $g:\Z^V \to \R\cup\set*{+\infty}$, $\delta \in (0,1]$, and $\varepsilon >0$.
  Given $T = \Omega\prn*{\prn*{\frac{C}{\varepsilon}}^2 \prn*{n + \log \frac{1}{\delta}}}$ i.i.d.\ draws of $g_1,\dots,g_T \sim \Dcal$, we can obtain $\phat \in \R^V$ that satisfies
	\[
		\E_{g \sim \Dcal} \brc*{\mubar(\phat; g)} \le \min_{\phat^* \in {[-C, +C]}^V} \E_{g \sim \Dcal} \brc*{\mubar(\phat^*; g)} + \varepsilon
	\]
  with probability at least $1-\delta$.
  Under the assumptions of \cref{theorem:learning}, we can compute $\phat$ in $\Ord(T \cdot (\Tineq + n^2))$ time.
\end{restatable}

\begin{remark}\label{rem:C}
  We can bound the constant, $C$, with parameters of minimization instances.
  In the bipartite-matching case, we have $C = nW$ if edge weights are always in $[-W, +W]$.
  See \citep[Section 4]{Sakaue2022-jr} for more information.
\end{remark}


\subsection{Basics of Online Gradient Descent}\label{sec:regred-bound}
We regard $f_t(\phat) = \mubar(\phat; g_t)$ as the $t$th loss for $t = 1,\dots,T$ and use the following standard OGD:
starting from $\phat_1 = \zeros$, in each $t$th round, play $\phat_t$, observe $f_t$, compute $z_t \in \partial f_t(\phat_t)$, and set $\phat_{t+1} \gets \Pi_C(\phat_{t} - \eta z_t)$, where $\eta>0$ is a learning rate and $\Pi_C$ is the $\ell_2$-projection onto $[-C,+C]^V$.
This OGD enjoys the following regret bound.

\begin{proposition}[{\citet[Section 2.2]{Orabona2020-dx}}]\label{prop:ogd}
  Let $C>0$ and $f_1,\dots,f_T$ be an arbitrary sequence of convex functions from $\R^V$ to $\R$.
  If OGD uses subgradients $z_t \in \partial f_t(\phat_t)$ such that $\norm{z_t}_2 \le L$ for $t=1,\dots,T$ and a learning rate of $\eta = \frac{C}{L}\sqrt{\frac{n}{T}}$, it returns $\phat_1,\dots,\phat_T$ satisfying
    \[
      \sum_{t=1}^T f_t(\phat_t) \le \min_{\phat^* \in {[-C, +C]}^V} \sum_{t=1}^T f_t(\phat^*) + CL\sqrt{nT}.
    \]
\end{proposition}
We prove \cref{theorem:learning} building on this proposition.
First, we confirm the convexity of the loss functions.

\begin{lemma}\label{lem:mubar-is-convex}
  $f_t(\phat) = \mubar(\phat; g_t)$ is convex in $\phat \in \R^V$.
\end{lemma}

\begin{proof}
  Let $S = \conv(\argmin g_t)$.
  We can rewrite $f_t(\phat) = \mubar(\phat; g_t)  = \min\Set*{\norm{p^* - \phat}_\infty^\pm}{p^*\in S}$ as
  \begin{align*}
    f_t(\phat)
    &= \inf\Set*{\norm{p^* - \phat}_\infty^\pm + \delta_{S}(p^*)}{p^*\in\R^V},
  \end{align*}
  where $\delta_{S}:\R^V \to \set*{0, +\infty}$ is the indicator function of a convex set $S$.
  Also, $\norm{\cdot}^\pm_\infty$ is convex by the triangle inequality.
  Thus, $f_t(\phat)$ is the infimal convolution of convex functions, hence convex \citep[Theorem 5.4]{Rockafellar1970-sk}.
\end{proof}

The following section completes the proof of \cref{theorem:learning} by presenting how to compute a subgradient $z_t$ of $f_t(\phat_t) = \mubar(\phat_t; g_t)$ such that $\norm{z_t}_2 \le L = \sqrt{2}$ in $\Tineq + \Ord(n^2)$ time.

\subsection{Computation of Subgradients}\label{subsec:subgradient}
We below omit $t$ and let, e.g., $g = g_t$ and $\phat = \phat_t$ for brevity since this section focuses only on the $t$th round.

First, we detail the assumption in \cref{theorem:learning}.
Recall that $\argmin g$ is an L-/\Ln-convex set due to \citep[Theorem~7.17]{Murota2003-bq}.
Therefore, $\conv(\argmin g)$ has an inequality-system representation as in \cref{proposition:l-convex-set}.
In this section, we assume one such inequality system to be available.
\begin{assumption}\label{assump:inequality-system}
  We can obtain an inequality-system representation of $\conv(\argmin g) \subseteq \R^V$ of the form
  \begin{align}\label{eq:inequality-system}
    \Set*{p \in \R^V}{\begin{alignedat}{2}
      &\text{$\alpha_i \le p_i \le \beta_i$ for $i \in V$}, \\
      &\text{$p_j - p_i \le \gamma_{ij}$ for distinct $i, j \in V$}
    \end{alignedat}}
  \end{align}
  in $\Tineq$ time, where $- \alpha_i, \beta_i, \gamma_{ij} \in \Z \cup \set{+\infty}$.
\end{assumption}

\begin{remark}
  Although inequality-system representations of $\conv(\argmin g)$ are not unique, whichever of the form~\eqref{eq:inequality-system} works in the following discussion.
  If an inequality system at hand lacks inequalities for some $i, j \in V$, we suppose those with
  $\alpha_i = -\infty$,
  $\beta_i = +\infty$, and
  $\gamma_{ij} = +\infty$
  to be given; we always apply this treatment to all $\alpha_i$ and $\beta_i$ if $g$ is L-convex since $\argmin g$ has no box constraints (see \cref{proposition:l-convex-set}).
\end{remark}

We then observe that computing the value of $\mubar(\phat; g)$ for any given $\phat \in \R^V$ can be reduced to a shortest path problem in a directed graph with possibly negative weights.
Since the reduction is presented in \citep[Appendix~D]{Sakaue2022-jr}, we here only give a brief description for later convenience.

Let $E = \Set*{ij}{i,j\in V; i\neq j}$ and $V_0 = \set{0} \cup V$.
We use $\Vtl = V_0 \cup \set*{s, t}$ as a vertex set, where $s$ is the origin and $t$ is the destination.
We define a set $\Etl$ of directed edges as
\[
  E\cup\set*{\set{0}\times V} \cup \set*{V\times\set{0}} \cup \set*{\set*{s} \times V_0} \cup \set*{V_0 \times \set*{t}}.
\]
Given any $\phat \in \R^V$, we define weights of edges $ij \in \Etl$ as
\begin{align}\label{eq:wtl}
  \wtl_{ij}(\phat) =
  \begin{cases}
    \gamma_{ij} - \phat_j + \phat_i & \text{if $ij \in E$,}
    \\
    -\alpha_i + \phat_i & \text{if $i \in V$ and $j = 0$,}
    \\
    \beta_j - \phat_j & \text{if $i = 0$ and $j \in V$,}
    \\
    0 & \text{if $i = s$ or $j = t$,}
  \end{cases}
\end{align}
where $\alpha_i$, $\beta_j$, $\gamma_{ij}$ are those representing $\conv(\argmin g)$ as in \eqref{eq:inequality-system}.
We take $ij \in \Etl$ to be removed if $\wtl_{ij} = +\infty$.

\begin{figure}
	\centering
	\begin{minipage}[b]{.4\linewidth}
			\centering
			\includegraphics[width=1.0\linewidth]{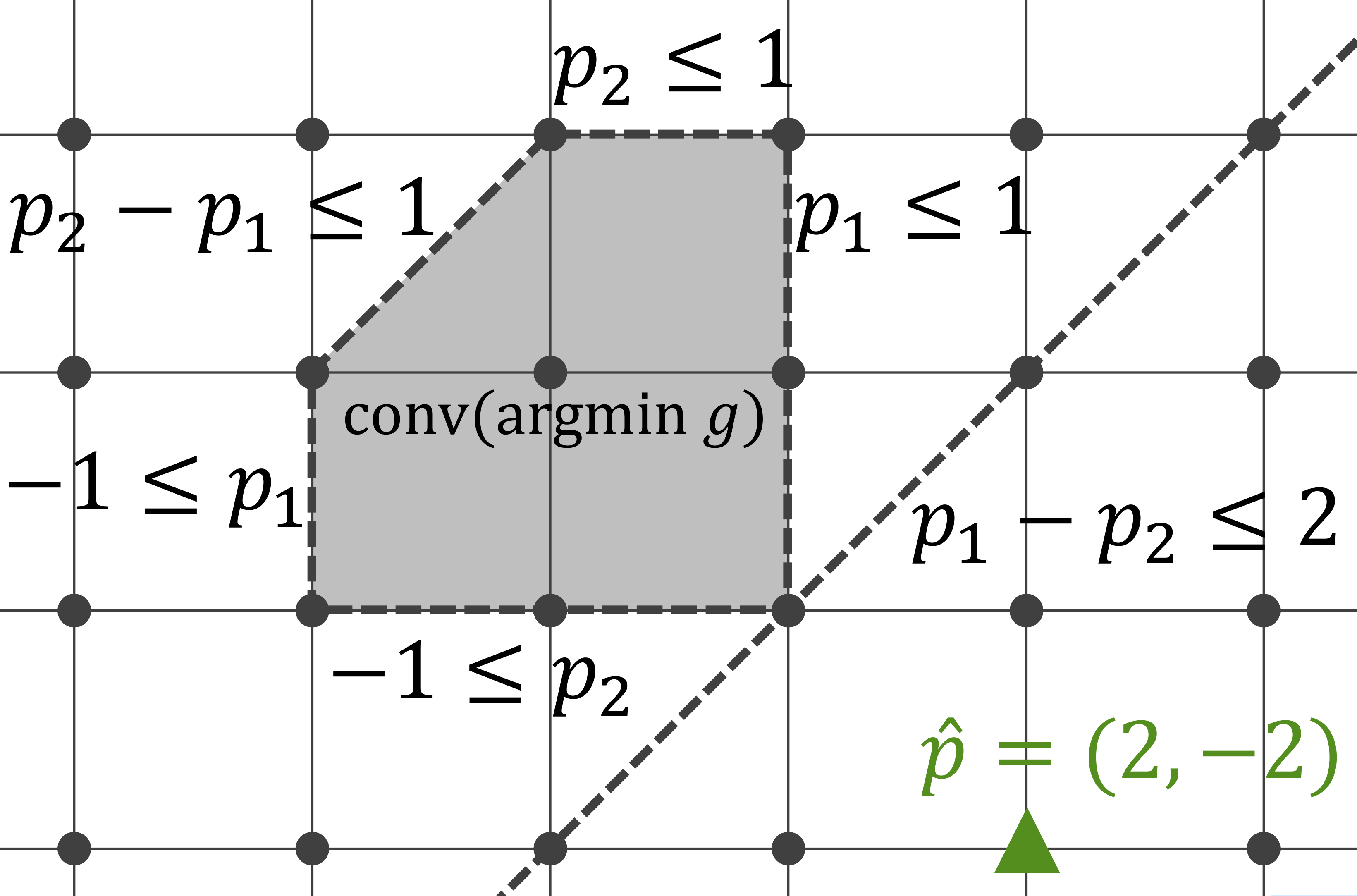}
			\subcaption{$\conv(\argmin g)$}\label{subfig:ln-convex-set}
	\end{minipage}
	\begin{minipage}[b]{.4\linewidth}
			\centering
			\includegraphics[width=1.0\linewidth]{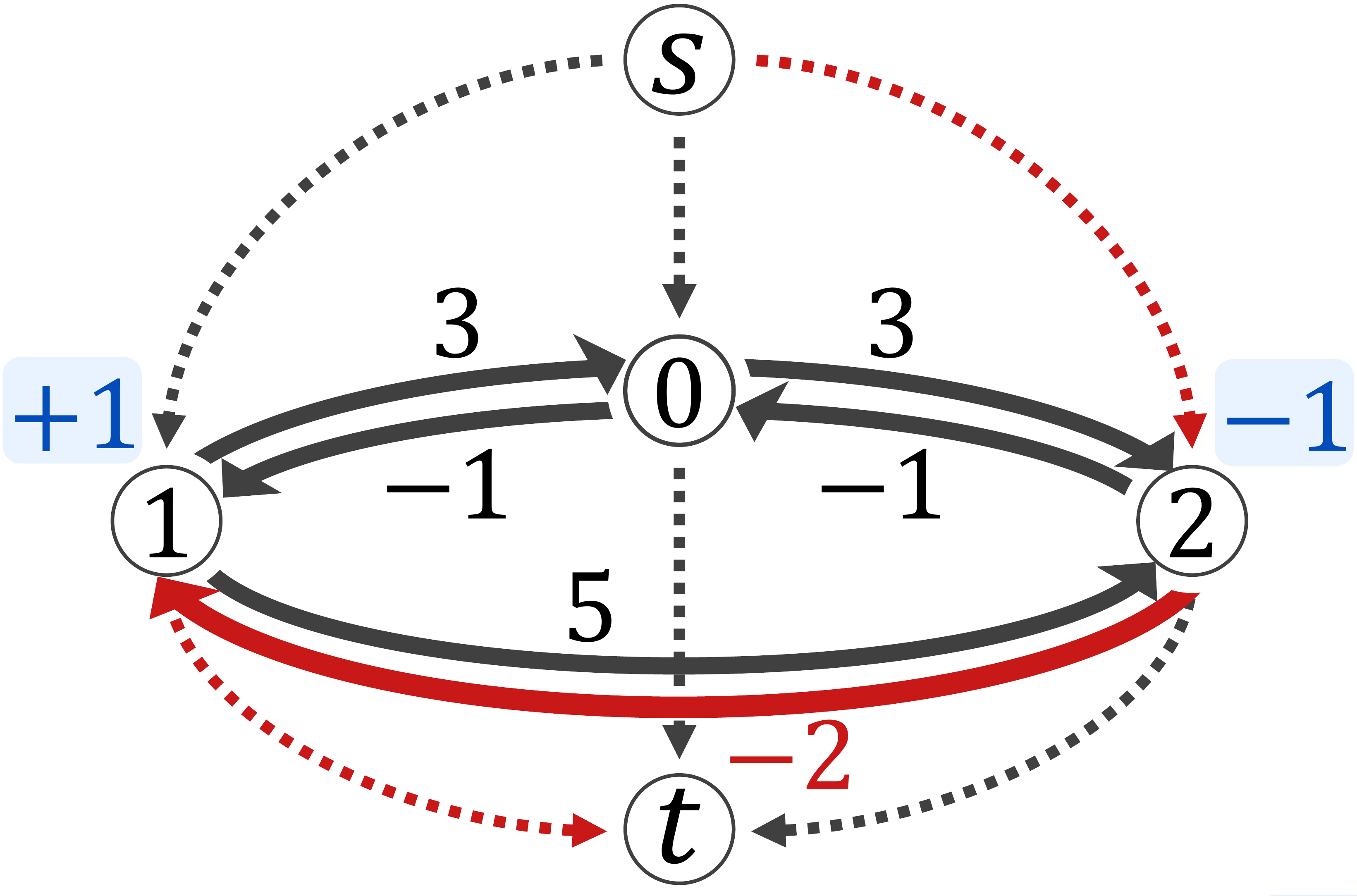}
			\subcaption{$(\Vtl, \Etl)$ and a shortest path}\label{subfig:shortest-path}
	\end{minipage}
	\caption{
    If $\conv(\argmin g)$ is given by inequalities in (a) and $\phat = (2, -2)$, we can compute $\mubar(\phat; g)$ by solving the shortest path problem in $(\Vtl, \Etl)$ as in (b), where weights of dashed edges are zero and the others have weights $\wtl_{ij}(\phat)$ shown nearby edges.
    A shortest path $\stpath^* = \set*{s2, 21, 1t}$ is shown in red, and the negative of its total weight is equal to $\mubar(\phat; g) = 2$.
    We can obtain a subgradient, $-\nabla \phi(\phat; \stpath^*) = (+1, -1)$, as shown in blue in (b).
    If we replace a redundant inequality constraint, $p_1 - p_2 \le 2$, in (a) with $p_1 - p_2 \le +\infty$, the edge from $2$ to $1$ is removed in (b); still, the other shortest path, $\set*{s2, 20, 01, 1t}$, yields the same subgradient.
	}
  \label{fig:reduction}
\end{figure}

Note that the negative weights, $-\wtl_{ij}(\phat)$, for $ij \in V_0 \times V_0$ indicate how much $\phat$ violates the corresponding inequalities in~\eqref{eq:inequality-system} representing $\conv(\argmin g)$.
From this fact, we can show that the negative of the total weight of a shortest $s$--$t$ path in $(\Vtl, \Etl)$ is equal to $\mubar(\phat; g)$, or how far $\phat$ is from $\conv(\argmin g)$ in terms of the $\ell^\pm_\infty$-norm (see \citep[Appendix~D]{Sakaue2022-jr}).
\Cref{fig:reduction} illustrates an example of $\conv(\argmin g)$ and the shortest path problem for computing $\mubar(\phat; g)$.
Note that $(\Vtl, \Etl)$ has no negative cycles; otherwise, the shortest-path weight is $-\infty$, hence $\mubar(\phat; g) = +\infty$, contradicting $\argmin g \ne \emptyset$ (\cref{assump:non-empty-argmin}).
Also, the shortest-path weight is always non-positive since there always exist zero-weight $s$--$t$ paths $\set{si, it}$ for $i \in V_0$.


We then rewrite $\mubar(\phat; g)$ keeping the reduction to the shortest path problem in mind.
Let $\Pcal \subseteq 2^{\Etl}$ be the set of all simple $s$--$t$ paths.
For each $\stpath \in \Pcal$, define $\phi(\cdot; \stpath):\R^V \to \R$ by
\begin{align}\label{eq:wtl-sum}
  \phi(\phat; \stpath) \coloneqq \sum_{ij \in \stpath} \wtl_{ij}(\phat),
\end{align}
which equals the total weight of an $s$--$t$ path $\stpath$.
Since $\mubar(\phat; g)$ is the negative of the total weight of a shortest path, we have
\begin{align}\label{eq:mu-is-linmax}
  \mubar(\phat; g)
  &= \max\Set*{-\phi(\phat; \stpath)}{\stpath \in \Pcal}.
\end{align}

Since each $-\phi(\phat; \stpath)$ is linear in $\phat$ by \eqref{eq:wtl} and \eqref{eq:wtl-sum}, and $\Pcal$ is finite (hence compact), Danskin's theorem \citep{Danskin1966-gb} (see, also \citep[Proposition B.22]{Bertsekas2016-qi}) implies
\[
  \partial \mubar(\phat; g) = \conv\Set*{-\nabla \phi(\phat; \stpath^*)}{\stpath^* \in \Pcal(\phat)},
\]
where $\Pcal(\phat) \coloneqq \argmax\Set*{-\phi(\phat; \stpath)}{\stpath \in \Pcal}$ is the set of all the shortest $s$--$t$ paths when $\phat \in \R^V$ is given.
Therefore, we can compute a subgradient of $\mubar(\cdot; g)$ at $\phat$ by finding a shortest $s$--$t$ path $\stpath^* \in \Pcal(\phat)$ and calculating
\[
  -\nabla \phi(\phat; \stpath^*) = - \sum_{ij \in \stpath^*} \nabla \wtl_{ij}(\phat).
\]
We then take a closer look at the subgradient $-\nabla \phi(\phat; \stpath^*)$.
From \eqref{eq:wtl}, each $-\nabla\wtl_{ij}(\phat)$ has at most one $-1$ and one $+1$.
These non-zeros are canceled out by taking the summation along the shortest path $\stpath^*$, except for at most two non-zeros, $-1$ and $+1$, corresponding to the two vertices adjacent to $s$ and $t$ in $\stpath^*$, respectively; if $s$ and/or $t$ are adjacent to $0 \in V_0$, the corresponding non-zeros also vanish.
See \cref{subfig:shortest-path} for an illustration of how $-\nabla \phi(\phat; \stpath^*)$ is calculated.

Formally, if the first and last edges in a shortest path $\stpath^* \in \Pcal(\phat)$ are $si$ and $jt$, respectively, with $i \neq j$, a subgradient $-\nabla \phi(\phat; \stpath^*) \in \partial \mubar(\phat; g)$ can be written as
\begin{align}\label{eq:subgradient}
  \begin{pNiceArray}{ccccc}[first-row]
     & \text{$i$th} &  & \text{$j$th}& \\
    0 \dots 0 & -\ind{i \neq 0} & 0 \dots 0 & \ind{j \neq 0} & 0 \dots 0
  \end{pNiceArray},
\end{align}
where $\ind{k \neq 0} = 1$ if $k \neq 0$ and $0$ otherwise; if $i = j$, the subgradient is zero.
To conclude, we obtain the next lemma.

\begin{lemma}\label{lem:subgradient}
  If an inequality system of $\conv(\argmin g)$ and $p^* \in \argmin g$ are available, we can compute a subgradient $z \in \partial \mubar(\phat; g)$ with $\norm{z}_2\le\sqrt{2}$ in $\Ord(n^2)$ time.
\end{lemma}

\begin{proof}
  Given a shortest $s$--$t$ path $\stpath^* \in \Pcal(\phat)$, we can compute $z \in \partial \mubar(\phat; g)$ as in \eqref{eq:subgradient}, which satisfies $\norm{z}_2 \le \sqrt{2}$.
  To obtain $\stpath^*$, we first transform the possibly negative edge weights \eqref{eq:wtl} into non-negative ones that preserve the shortest-path set $\Pcal(\phat)$ via a \emph{potential} (see \cref{asec:potentials} for details).
  We can do this transformation in $\Ord(|\Etl|)$ time by using an optimal solution $p^* \in \argmin g$, which is assumed to be given in \cref{theorem:learning}.
  Therefore, by finding a shortest $s$--$t$ path $\stpath^*$ with Dijkstra's algorithm in $\Ord(|\Etl| + |\Vtl|\log|\Vtl|) \lesssim \Ord(n^2)$ time, we can compute a subgradient in $\Ord(n^2)$ time.
\end{proof}

We can easily obtain an intuition of the subgradient \eqref{eq:subgradient} when a shortest $s$--$t$ path, $\stpath^* \in \Pcal(\phat)$, is of the form $\set*{si, ij, jt}$.
Such a shortest path implies that a current prediction $\phat$ violates inequality constraint $p_j - p_i \le \gamma_{ij}$ most largely among those in \eqref{eq:inequality-system} that represent $\conv(\argmin g)$.
Updating $\phat$ along the negative direction of the subgradient~\eqref{eq:subgradient} reduces the magnitude of the violation, $\phat_j - \phat_i - \gamma_{ij} >0$, by increasing $\phat_i$ and decreasing $\phat_j$.
Thus, the subgradient descent moves a current prediction closer to $\conv(\argmin g)$.

At a high level, our strategy is to write the distance to the set of optimal solutions as a maximum of linear (or convex) functions, as in \eqref{eq:mu-is-linmax}, and use Danskin's theorem to obtain a subgradient.
Then, we can use OGD to learn predictions close to sets of optimal solutions.
We expect that this simple idea is also useful in other settings, e.g., \citep{Chen2022-li}.

\section{Obtaining Inequality Systems of Minimizers}\label{sec:minimizer-set}
We show how to get an inequality system of $\conv(\argmin g)$ as in \cref{assump:inequality-system}.
\Cref{subsec:efficient-ineq-method} provides efficient methods that utilize problem-specific structures, and \cref{subsec:general-ineq-method} presents a general polynomial-time method that only uses black-box access to $g$, implying $\Tineq$ is at most polynomial.

\subsection{Efficient Problem-Specific Methods}\label{subsec:efficient-ineq-method}
We can efficiently construct a desired inequality system if the primal-dual structure of the problem, $\min_{p \in \Z^V} g(p)$, is available.
We detail this method for bipartite matching.

We consider the weighted perfect bipartite matching problem introduced in \cref{sec:introduction}.
Let $(V, E)$ be a bipartite graph with equal-sized bipartition $V = L \cup R$, weights $w \in \Z^E$, $n = |V|$, and $m = |E|$.
Recall that we can write the dual LP as in~\eqref{prob:matching-dual} with constraints $s_i - t_j \ge w_{ij}$ for $ij \in E$.
The following complementarity theorem gives a useful characterization of the set of dual optimal solutions.

\begin{proposition}[{\citet[Proposition~2.3]{Murota1995-bo}}]\label{thm:matching-complementary}
  Let $M \subseteq E$ be a matching in $(V, E)$ and $p = (s, t) \in \Z^{L\cup R}$ a dual feasible solution to \eqref{prob:matching-dual}.
  Then, $M$ and $p$ are optimal if and only if $s_i - t_j = w_{ij}$ for all $ij \in M$.
\end{proposition}
This proposition implies that given an arbitrary maximum weight matching $M^* \subseteq E$, we can represent the set of dual optimal solutions, $\argmin g$, by an inequality system as
\[
  \Set*{p = (s, t) \in \Z^{L \cup R}}{\begin{alignedat}{2}
    &\text{$s_i - t_j \ge w_{ij}$ for $ij \in E$}, \\
    &\text{$s_i - t_j \le w_{ij}$ for $ij \in M^*$} \\
  \end{alignedat}},
\]
and replacing $\Z^{L \cup R}$ with $\R^{L \cup R}$ yields an inequality-system representation of $\conv(\argmin g)$ (see \cref{proposition:l-convex-set}).
Note that a maximum weight matching $M^*$ is usually available for free since the $t$th instance is already solved when learning $\phat_t$.
Once $M^*$ is given, we can construct the above inequality system in $\Ord(m)$ time, hence $\Tineq = \Ord(m)$.

Having seen $\Tineq$ is small enough, the dominant part in the per-round time complexity of OGD is Dijkstra's algorithm for computing a subgradient, which runs in $\Ord(m + n\log n)$ time since the above inequality system leads to graph $(\Vtl, \Etl)$ with $|\Etl| = \Ord(m)$.
Hence, OGD's per-round running time is shorter than that of solving local optimization in \cref{alg:steepest-descent} once with the $\Ord(m\sqrt{n})$-time Hopcroft--Karp algorithm.


The core idea of the above method is to utilize the ``if and only if'' condition of the complementarity theorem.
In other words, once we find an arbitrary primal (dual) optimal solution, we can capture the set of all dual (primal) optimal solutions via the complementarity condition.
This idea has been well studied in \emph{combinatorial relaxation} \citep{Murota1995-bo} and is applicable to matroid intersection and discrete energy minimization.
We below present the results; see \cref{asec:proof-mi-argmin,asec:proof-dem-argmin}, respectively, for the proofs.

\begin{restatable}{theorem}{miargmin}\label{thm:mi-argmin}
  Consider the dual problem, $\min_{p \in Z^V} g(p)$, of weighted matroid intersection defined on a ground set $V$ of size $n$.
  If a maximum weight common base is available (or the problem is already solved), we can obtain an inequality system of the form \eqref{eq:inequality-system} representing $\conv(\argmin g)$ in $\Tineq = \Ord(\tau nr)$ time, where $r$ is the rank of the matroids and $\tau$ is the running time of independence oracles.
\end{restatable}

\begin{restatable}{theorem}{demargmin}\label{thm:dem-argmin}
  Consider discrete energy minimization, $\min_{p \in Z^V} g(p)$, defined on a graph with $n$ vertices and $m$ edges.
  If $\dom g \subseteq [0, W]^V$ for some $W>0$ (which is true in most computer-vision applications), we can obtain an inequality system of the form \eqref{eq:inequality-system} representing $\conv(\argmin g)$ in $\Tineq = \Ord(mn\log(n^2/m)\log(nW))$ time.
\end{restatable}


\subsection{General Polynomial-Time Method}\label{subsec:general-ineq-method}
We then discuss L-/\Ln-convex minimization, $\min_{p \in \Z^V} g(p)$, where we only have black-box access to $g$ values.
Unlike the above cases, this setting does not enjoy useful primal-dual structures.
Still, we can construct a desired inequality system in polynomial time.
We here assume $\argmin g \cap [-C, +C]^V \neq \emptyset$, where $C>0$ is the constant used in OGD, to deal with possibly unbounded $\argmin g$.
This condition is reasonable since the best prediction, $\phat^*$, in \cref{theorem:learning} is selected from $[-C, +C]^V$.
We also assume $g$ to have a finite minimum value.
Under these assumptions, the following theorem holds (see \cref{asec:find-ineq-general} for the complete proof).

\begin{restatable}{theorem}{genargmin}\label{thm:gen-argmin}
  For general L-/\Ln-convex function minimization, $\min_{p \in Z^V} g(p)$, such that $\argmin g \cap [-C, +C]^V \neq \emptyset$ and $\min g > -\infty$, we can obtain an inequality system of a subset of $\conv(\argmin g)$ that is sufficient for the subgradient computation in $\Tineq = \Ord(n^2\log^2 C \cdot (\mathrm{EO}\cdot n^3\log^2n + n^4\log^{\Ord(1)}n))$ time, where $\mathrm{EO}$ is the time for evaluating $g$.
\end{restatable}

\begin{proof}[Proof sketch]
  Note that $\argmin g$ can be written with $\Ord(n^2)$ inequalities due to \cref{proposition:l-convex-set}.
  We seek appropriate values of all the $\Ord(n^2)$ constants, $\alpha_i, \beta_i, \gamma_{ij} \in \Z$, via binary search, each of which takes $\Ord(\log C)$ iterations by the assumption of $\argmin g \cap [-C, +C]^V \neq \emptyset$.
  In each iteration, we check whether a given inequality, e.g., $p_j - p_i \le \gamma_{ij}$, is satisfied by all relevant minimizers of $g$ or not.
  Based on the steepest descent scaling algorithm \citep[Section~10.3.2]{Murota2003-bq}, we can check this by solving submodular function minimization (defined as with local optimization in \cref{step:local-opt} of \cref{alg:steepest-descent}) $\Ord(\log C)$ times.
  If we solve it with an $\Ord(\mathrm{EO}\cdot n^3\log^2n + n^4\log^{\Ord(1)}n)$-time algorithm of \citep{Lee2015-ia}, we obtain the desired time complexity.
\end{proof}


\section{Empirical Observation}
We experimentally compared our learning method with the previous methods \citep{Dinitz2021-sd,Sakaue2022-jr}, which learn prediction $\phat$ based on $\norm{p^* - \phat}_1$ and $\norm{p^* - \phat}_\infty$, respectively, defined with some optimal $p^*$; we let $p^*$ be outputs of warm-started algorithms.
We applied those learning methods to small random bipartite matching instances and compared the number of iterations of the Hungarian method (or \cref{alg:steepest-descent} for bipartite matching) warm-started with learned predictions.
Our learning method produced predictions that led to fewer iterations of the Hungarian method than those of the previous methods.
This result suggests that predictions learned by minimizing the $\ell^\pm_\infty$-distance to the set of optimal solutions can be more beneficial than those learned with some fixed optimal solutions.
In addition, our learning method converged to good predictions more quickly, implying that it empirically requires fewer sampled instances to learn good predictions.
Those observations suggest that our learning method is not merely of theoretical interest.
We present the details of the experiments in \cref{asec:experiment}.

\section{Conclusion}\label{section:conclusion}
We have presented a new warm-start-with-prediction framework for L-/\Ln-convex minimization that provides time complexity bounds proportional to $\mubar(\phat; g)$, the $\ell_\pm$-distance between a prediction $\phat$ and the set, $\conv(\argmin g)$, of optimal solutions.
Specifically, we have shown that the steepest descent method warm-started by $\phat$ takes $\Ord(\mubar(\phat; g))$ iterations and that we can learn $\phat$ to approximately minimize $\E_g[\mubar(\phat; g)]$ in polynomial time.
At a technical level, we have shown an efficient method for computing subgradients of $\mubar(\cdot; g)$ to learn $\phat$ with OGD.
Our results imply the first polynomial-time learnability of predictions that can provably warm-start algorithms regardless of the non-uniqueness of optimal solutions.
This implication would be significant progress in warm-starts with predictions because the non-uniqueness always exists in the broad class of L-/\Ln-convex minimization, as described in \cref{sec:introduction} and \cref{rem:multiple-opt}.
Studying how to learn predictions with similar guarantees for other problems will be an interesting future direction.

\subsection*{Acknowledgements}
This work was supported by JST ERATO Grant Number JPMJER1903 and JSPS KAKENHI Grant Number JP22K17853.

\bibliographystyle{abbrvnat}
\bibliography{DCAWarmStart}

\appendix

\clearpage

\begin{center}
	{\fontsize{18pt}{0pt}\selectfont \bf Appendix}
\end{center}

\section{Details of Issues Caused by Non-uniqueness of Optimal Solutions}\label{asec:remark}
We detail the problem caused by ignoring the non-uniqueness of optimal solutions $p^*$, which arises in many problem settings, including the dual of bipartite matching and L-/\Ln-convex function minimization.
In short, the problem is a kind of dilemma:
if we fix some optimal $p^*$ independently of prediction $\phat$, the time complexity bounds depending on $\norm{p^* - \phat}$ can be poor;
if~we select optimal $p^*$ depending on prediction $\phat$ to make $\norm{p^* - \phat}$ small, we cannot use existing results on the learnability of $\phat$ due to the dependence of $p^*$ on $\phat$.
We below detail these two types of problems.

\subsection{On Fixing Optimal Solutions Independently of Predictions}
If an optimal solution is uniquely associated with each input instance independently of predictions, we can rely on the existing learnability results of predictions.
Hence, one may first think of using some tie-breaking rule to handle the non-uniqueness of optimal solutions.
Note, however, that no tie-breaking rule can lead to essentially stronger results than ours of using the minimum distance to the set of all optimal solutions.
Moreover, seemingly reasonable tie-breaking rules often result in poor bounds.
For example, consider a natural tie-braking rule that uniquely selects an optimal $p^*$ closest to some fixed point, say the origin $\zeros \in \R^V$, in the $\ell_2$-norm.
If an instance with $\conv(\argmin g) = \Set*{p \in \R^2}{ p_2 - p_1 = 100}$ is given, $p^* \in \conv(\argmin g)$ closest to $\zeros$ is $p^* = (-50, 50)$.
Then, if a given prediction is $\phat = (1, 100)$, we have $\norm{p^* - \phat}_\infty = 50$ even though $\mubar(\phat; g) = 1$.
Therefore, such a tie-breaking rule that selects $p^*$ closest to some fixed point generally results in poor prediction-dependent time complexity bounds even when $\phat$ is close to the set, $\conv(\argmin g)$, of optimal solutions.

\subsection{On Existing Results for Learning Predictions}\label{asec:existing-learnability}
Considering the above drawback of fixing $p^*$, one may want to let $p^*$ be an optimal solution close to a given prediction $\phat$.
Indeed, most experiments in the previous studies seem to be implicitly based on this kind of idea; that is, they let $p^*$ be an optimal solution returned by an algorithm warm-started by $\phat$ and learn $\phat$ to decrease loss values of the form $\norm{p^* - \phat}$.
This idea, however, makes optimal $p^* \in \conv(\argmin g)$ selected depending on a given prediction $\phat$.
We below explain why the existing theoretical results for learning predictions $\phat$ cannot deal with the dependence of $p^*$ on $
\phat$.

\paragraph*{PAC learning approach.}
A popular approach to obtaining guarantees for learning predictions is to use the PAC learning framework \citep{Dinitz2021-sd,Chen2022-li}.
With this approach, supposing input instances $g$ to be drawn i.i.d.\ from a distribution, we usually analyze the \emph{pseudo-dimension} of a function class of the form $\Set*{f_{\phat}: g \mapsto \R}{\phat\in\R^V}$.
The existing studies, however, analyzed the pseudo-dimension of a class of functions of the form $f_{\phat}(p^*) = \norm{p^* - \phat}$, ignoring the non-uniqueness of optimal $p^*$.
If we want to let $p^*$ be an optimal solution close to $\phat$, we must regard $p^*$ as a function of $\phat$ and specify how $p^*$ is uniquely computed from $\phat$ for each instance $g$; hence, the function should look like $f_{\phat}(g) = \norm{p^*(\phat, g) - \phat}$.
No existing PAC learnability results for warm-starts with predictions have discussed such a complicated dependence.

\paragraph*{Online learning approach.}
Another approach is to use online algorithms for learning predictions \citep{Khodak2022-sf,Sakaue2022-jr}.
In those studies, the $t$th loss function takes the form $f_t(\phat) = \norm{p^*_t - \phat}$, where $p^*_t$ is some optimal solution selected for the $t$th instance, and the regret is defined as $\sum_{t = 1}^T \norm{p^*_t - \phat_t} - \min_{\phat^* \in {[-C, +C]}^V} \sum_{t = 1}^T \norm{p^*_t - \phat^*}$.
If we want to let $p^*$ depend on $\phat$, we need to learn $\phat_t$ against an adversary who selects $p^*_t$ depending on $\phat_t$.
In this situation, the above regret does not make sense; one obvious issue is that there is room for achieving a small regret by choosing $\phat_t$ to make the second term large since $\phat_t$ can affect $p^*_t$.
A similar issue is discussed in \citep{Arora2012-hp}, but the problem here would be more severe since the adversary acts \emph{after} the learner.
The existing studies have not considered this situation.
Note that, although our method belongs to this category, our loss function, $\mubar(\phat; g_t)$, is designed to avoid the non-uniqueness issue.

\subsection{On Learnability Result of \texorpdfstring{\citep{Polak2022-je}}{(Polak \& Zub, 2022)}}\label{asubsec:polak}
\citet{Polak2022-je} have studied a maximum flow algorithm warm-started with predictions.
The authors have stated that their method enjoys a time complexity bound proportional to $\norm{p^* - \phat}_1$ for an optimal flow $p^*$ closest to $\phat$.
Although this means that $p^*$ depends on $\phat$, their analysis for learning $\phat$ seems insufficient for handling the dependence, as detailed below.

In \citep[Lemma~7]{Polak2022-je}, the authors present a uniform bound on the difference between empirical and expected losses.
Specifically, given $T$ instances $g_1,\dots,g_T$ drawn i.i.d.\ from a distribution $\Dcal$, the lemma says that a bound of the form
$
 \abs*{
  \frac{1}{T}\sum_{t=1}^T \norm{p^*(g_t) - p}_1 - \E_{g\sim\Dcal} \brc*{ \norm{p^*(g) - p}_1}
 }
 \le 1
$
holds for all $p \in \Z^V$ satisfying some constraints with high probability.
Then, in the proof of \citep[Theorem~4]{Polak2022-je}, the authors use Lemma~7 substituting prediction $\phat$ into $p$, where the ``for all $p \in \Z^V$'' part is justified by using the fact that prediction $\phat$ is chosen after instances $g_t$ are sampled from $\Dcal$.
This justification is correct if $p^*(g)$ is independent of $\phat$, i.e., we can uniquely define function $f_{g}(p) = \norm{p^*(g) - p}_1$ from $g$.
However, if we let $p^*(g)$ be an optimal solution closest to $\phat$, this justification is incorrect.
For a bound like Lemma~7 to hold for $p^*(g)$ depending on $\phat$, we need to derive a uniform convergence by, e.g., defining how $p^*(g)$ is computed from a pair of $(\phat, g)$ and bounding the pseudo-dimension for the computation procedure.
Considering the above, the sample complexity bound of \citep{Polak2022-je} for learning $\phat$ seems to be true only when $p^*(g)$ is fixed for each $g$ independently of $\phat$.

\section{Proof of \texorpdfstring{\Cref{lem:muismubar}}{Lemma \ref{lem:muismubar}}}\label{asec:proof-of-lem-musimubar}
\muismubar*

\begin{proof}
  Let $p \in \Z^V$.
  As discussed in \cref{subsec:subgradient}, $\mubar(p; g)$ is equal to the negative of the total weight of a shortest path in $(\Vtl, \Etl)$, which we can represent as an optimal value of the following LP (see \citep[Appendix~D]{Sakaue2022-jr}):
  \begin{align}\label{app-problem:projection-dual-lp}
    \begin{array}{lllll}
      &\displaystyle\maximize \quad &{q}_t - {q}_s &
      \\
      &\subto \quad &{q}_j - {q}_i \le \wtl_{ij}(p) & \forall ij \in \Etl,
      \\
      & & {q}_0 = 0, &
    \end{array}
  \end{align}
  where
  \begin{align*}
    \wtl_{ij}(p) =
    \begin{cases}
      \gamma_{ij} - p_j + p_i & \text{if $ij \in E$,}
      \\
      -\alpha_i + p_i & \text{if $i \in V$ and $j = 0$,}
      \\
      \beta_j - p_j & \text{if $i = 0$ and $j \in V$,}
      \\
      0 & \text{if $i = s$ or $j = t$}
    \end{cases}
  \end{align*}
  for $-\alpha_i, \beta_j, \gamma_{ij} \in \Z \cup \set{+\infty}$.
  From $p \in \Z^V$, all $\wtl_{ij}(p)$ are integers.
  Furthermore, the constraints in \eqref{app-problem:projection-dual-lp} can be written with a totally unimodular matrix.
  Therefore, the LP \eqref{app-problem:projection-dual-lp} has an integer optimal solution $q^* \in \Z^{\Vtl}$.
  Let $q^*_V \in \Z^V$ be the restriction of $q^*$ to $V = \Vtl\setminus\set*{0, s, t}$.
  Then, $p^* = p + q^*_V$ attains $\norm{p^* - p}^\pm_\infty = \mubar(p; g)$, as shown in \citep[Appendix~D]{Sakaue2022-jr}.
  Moreover, $p^*$ is an integer vector due to $p\in \Z^V$ and $q^*_V \in \Z^V$, and thus we have $p^* \in \conv(\argmin g) \cap \Z^V = \argmin g$, where the equality is known as the \emph{hole-free property} of L-/\Ln-convex sets \citep[Theorem 5.2 and Section 5.5]{Murota2003-bq}.
  From $\norm{p^* - p}^\pm_\infty = \mubar(p; g)$ and $p^* \in \argmin g$, we have $\mu(p; g) = \norm{p^* - p}^\pm_\infty$, hence $\mu(p; g) = \norm{p^* - p}^\pm_\infty = \mubar(p; g)$.
\end{proof}

\section{Proof of \texorpdfstring{\Cref{cor:sample}}{Corollary \ref{cor:sample}}}\label{asec:proof-of-cor-sample}
\sample*

\begin{proof}
  The basic proof idea is to use online-to-batch conversion \citep{Cesa-Bianchi2004-id} to convert the regret bound (\cref{theorem:learning}) into the sample complexity bound.
  Here, we use a refined variant, called anytime online-to-batch conversion \citep[Theorem~1]{Cutkosky2019-ik}, which is useful for ensuring the last iterate convergence of predictions computed for stochastic loss functions.
  We will also use this technique in the experiments in \cref{asec:experiment}.

  To use \citep[Theorem~1]{Cutkosky2019-ik}, we slightly modify the online algorithm for computing predictions.
  In each $t$th round, let $q_t = \frac{1}{t} \sum_{t^\prime = 1}^t \phat_{t^\prime}$ and compute a subgradient $z_t$ at $q_t$, i.e., $z_t \in \partial \mubar(q_t, g_t)$.
  The $t$th loss function revealed to an online learner is a linear loss function $f_t(\phat) = \iprod{z_t, \phat}$, and the learner uses an online algorithm to compute $\phat_1,\dots,\phat_T \in [-C, +C]^V$ that satisfy a regret bound of $\sum_{t=1}^T f_t(\phat_t) - \sum_{t=1}^T f_t(\phat^*) = \sum_{t = 1}^T \iprod{z_t, \phat_t - \phat^*} = \Ord(C\sqrt{nT})$ for any $\phat^* \in [-C, +C]^V$.
  Here, the learner can use OGD described in \cref{sec:regred-bound}; one can easily confirm that it enjoys the $C\sqrt{2nT}$-regret bound also for the linearized loss $f_t(\phat) = \iprod{z_t, \phat}$.
  Then, by substituting $\norm{z_t}_1 \le 2$ and $\max\Set*{\norm{p - q}_\infty}{p, q \in [-C, +C]^V} \le 2C$ into \citep[Theorem~1]{Cutkosky2019-ik}, we can show that $\phat = q_T$ satisfies the following inequality with a probability of at least $1-\delta$:
  \begin{align*}
    \E_{g \sim \Dcal} \brc*{\mubar(\phat; g_t)} - \min_{\phat^* \in {[-C, +C]}^V} \E_{g \sim \Dcal} \brc*{\mubar(\phat^*; g_t)} \le \frac{C\sqrt{2nT} + 8C\sqrt{T\log(2/\delta)}}{T}
    =
    \frac{C}{\sqrt{T}}\prn*{ \sqrt{2n} + 8\sqrt{\log\frac{2}{\delta}} }.
	\end{align*}
	Thus, the sample size of $T = \Omega\prn*{\prn*{\frac{C}{\varepsilon}}^2 \prn*{n + \log \frac{1}{\delta}}}$ is sufficient for ensuring that the right-hand side is at most $\varepsilon$.

  As for the time complexity, the per-round running time of OGD is $\Tineq + \Ord(n^2)$ by \cref{lem:subgradient} (since $p^*_t \in \argmin g_t$ is available), and this is repeated $T$ times to obtain $\phat = q_T$.
  Therefore, it takes $\Ord(T \cdot (\Tineq + n^2))$ time in total.
\end{proof}

\section{Regret Lower Bound}\label{asec:lower-bound}
We show an $\Omega(C\sqrt{nT})$ regret lower bound for online minimization of $\mubar(\phat; g_t)$ to complement \cref{theorem:learning}.
The proof idea is based on \citep[Theorem 14]{Hazan2012-hc}, which presents a regret lower bound for online submodular minimization.

For ease of analysis, we only consider a learner who selects $\phat_1,\dots,\phat_T$ from $[-C, +C]^V$.
Our OGD satisfies this condition due to the $\ell_2$-projection onto $[-C, +C]^V$; hence the lower bound implies the tightness of the $\Ord(C\sqrt{nT})$ upper bound.
Another remark is that we below obtain a lower bound by using $g_t$ such that $\argmin g_t \cap [-C, +C]^V = \emptyset$, while predictions $\phat$ are constrained to $[-C, +C]^V$.
We leave it for future work to prove a lower bound using $g_t$ with $\argmin g_t \cap [-C, +C]^V \neq \emptyset$.

\begin{theorem}
  Let $C>0$ be an integer.
  For any online leaner who plays $\phat_1,\dots,\phat_T \in [-C, +C]^V$, there is a sequence of \Ln-convex functions $g_1,\dots,g_T$ such that the learner incurs an $\Omega(C\sqrt{nT})$ regret.
\end{theorem}

\begin{proof}
Let $n = |V|$ be even and $i(t) = (t \bmod n/2) + 1 \in \set{1,\dots,n/2}$ for $t = 1,\dots,T$.
In each $t$th round, choose a Rademacher random variable $\sigma_t \in \set{-1, +1}$ independently of all other random variables.
Let $g_t$ be an indicator function such that $\dom g_t$ is a singleton, $\set{p^*_t}$, where $p^*_t \in \Z^V$ has two non-zeros:
the $i(t)$th entry is $3\sigma_tC$,
the ($i(t) + n/2$)th entry is $-3\sigma_tC$,
and the others are zero.
Since we have $\argmin g_t = \set{p^*_t}$, for any $\phat \in [-C, +C]^V$, it holds that
\[
  \mubar(\phat; g_t)
  = \norm{p^*_t - \phat}_\infty^\pm
  = \max_{i\in V}\max\set*{0, p^*_{t, i} - \phat_i} + \max_{i\in V}\max\set*{0, \phat_i - p^*_{t, i}}
  = 6C + \sigma_t(\phat_{i(t) + n/2} - \phat_{i(t)}).
\]
Thus, for any learner's choice $\phat_t \in [-C, +C]^V$, we have $\E[\mubar(\phat_t; g_t)] = 6C$, where the expectation is taken over the randomness of $\sigma_t$.
Therefore, the expected total loss of any online learner is $6CT$.

We then show that there exists $\phat^* \in [-C, +C]^V$ that has an $\Omega(C\sqrt{nT})$ advantage over the learner's expected loss, implying an $\Omega(C\sqrt{nT})$ regret lower bound.
Let $\sign(x)$ denote a function that returns $+1$ if $x > 0$, $0$ if $x = 0$, and $-1$ if $x < 0$.
Let $X_i = \sum_{t: i(t) = i} \sigma_t$ for $i = 1,\dots,n/2$.
Set the $i$th entry of $\phat^*$ to $\sign(X_i) \times C$ for $i = 1,\dots,n/2$ and $-\sign(X_i) \times C$ for $i=n/2+1,\dots,n$.
Then, in each $t$th round, the $i(t)$th entry of $p^*_t - \phat^*$ causes a loss value of
$2C$ if $\sign(X_{i(t)}) = \sigma_t$,
$3C$ if $X_{i(t)} = 0$, and
$4C$ otherwise. 
Similarly, the ($i(t)+n/2$)th entry causes a loss value of $2C$, $3C$, or $4C$.
Hence we have
\[
  \sum_{t = 1}^T \mubar(\phat^*; g_t)
  = \sum_{t = 1}^T \norm{p^*_t - \phat^*}_\infty^\pm
  = 3CT - C\sum_{i = 1}^{n/2}|X_i| + 3CT - C\sum_{i = 1}^{n/2}|X_i| = 6CT - 2C\sum_{i=1}^{n/2} |X_i|,
\]
which implies that the expected regret is at least $2C\times \E\brc*{ \sum_{i=1}^{n/2} |X_i| }$.
Since each $X_i$ is a sum of at least $\floor*{\frac{T}{n/2}}$ independent Rademacher random variables, Khintchine's inequality (see, e.g., \citep[Appendix~A.1.4]{Cesa-Bianchi2006-oa}) implies $\E\brc*{ |X_i| } \ge \sqrt{\frac{1}{2}\floor*{\frac{T}{n/2}}}$.
Thus, the expected regret is at least $2C\times \frac{n}{2} \sqrt{\frac{1}{2}\floor*{\frac{T}{n/2}}} = \Omega(C\sqrt{nT})$.
This expected lower bound implies that there is a specific choice of $\sigma_t$ values such that the learner incurs an $\Omega(C\sqrt{nT})$ regret.
\end{proof}

\section{Transformation into Non-negative Edge Weights}\label{asec:potentials}
We show how to transform the shortest path problem in \cref{subsec:subgradient} into another one with non-negative edge weights.
Once we obtain such a transformed problem, we can use Dijkstra's algorithm to find a shortest path.
The transformation is based on a so-called \emph{potential}, which has been well studied in combinatorial optimization \citep[Section~8.2]{Schrijver2003-ol}.

Recall that the vertex set is $\Vtl = V_0 \cup \set*{s, t}$, where $V_0 = \set{0} \cup V$, and the edge set is
\begin{align*}
  \Etl = E\cup\set*{\set{0}\times V} \cup \set*{V\times\set{0}} \cup \set*{\set*{s} \times V_0} \cup \set*{V_0 \times \set*{t}},
\end{align*}
where $E = \Set*{ij}{i,j\in V; i\neq j}$.
The set of all simple $s$--$t$ paths in $(\Vtl, \Etl)$ is denoted by $\Pcal \subseteq 2^{\Etl}$.
We also have the following inequality-system representation of $\conv(\argmin g)$, as in \cref{assump:inequality-system}:
\begin{align}\label{aeq:inequality-system}
  \conv(\argmin g)
  =
  \Set*{p \in \R^V}{
    \begin{alignedat}{2}
    &\text{$\alpha_i \le p_i \le \beta_i$ for $i \in V$}, \\
    &\text{$p_j - p_i \le \gamma_{ij}$ for distinct $i, j \in V$}
    \end{alignedat}
  }.
\end{align}

For any given prediction $\phat \in \R^V$, the original (possibly negative) edge weights $\wtl_{ij}(\phat)$ ($ij\in\Etl$) are defined as follows:
\begin{align*}
  \wtl_{ij}(\phat) =
  \begin{cases}
    \gamma_{ij} - \phat_j + \phat_i & \text{if $ij \in E$,}
    \\
    -\alpha_i + \phat_i & \text{if $i \in V$ and $j = 0$,}
    \\
    \beta_j - \phat_j & \text{if $i = 0$ and $j \in V$,}
    \\
    0 & \text{if $i = s$ or $j = t$.}
  \end{cases}
\end{align*}
We transform them into non-negative weights.
We call $q \in \R^{\Vtl}$ a \emph{potential} if $\wtl_{ij}(\phat) - q_j + q_i \ge 0$ holds for $ij \in \Etl$.
If we have a potential, we can define non-negative edge weights $\wtl^+_{ij}(\phat) \coloneqq \wtl_{ij}(\phat) - q_j + q_i$ for $ij \in \Etl$.
For any simple $s$--$t$ path $\stpath \in \Pcal$ in $(\Vtl, \Etl)$, the telescoping sum implies
\[
  \sum_{ij\in \stpath} \wtl^+_{ij}(\phat) = \sum_{ij\in S} (\wtl_{ij}(\phat) - q_j + q_i) = q_s - q_t + \sum_{ij \in S} \wtl_{ij}(\phat),
\]
where $q_s$ and $q_t$ are independent of the choice of $\stpath \in \Pcal$.
Hence $\stpath \in \Pcal$ is the shortest with respect to edge weights $\wtl_{ij}(\phat)$ if and only if $\stpath$ is the shortest with respect to $\wtl^+_{ij}(\phat)$.
Therefore, once a potential is given, we can obtain non-negative edge weights that do not change the set of shortest paths in $\Ord(\Etl)$ time, and we can find a shortest path with Dijkstra's algorithm in $\Ord(|\Etl| + |\Vtl|\log |\Vtl|)$ time.
We below present how to obtain a potential from an arbitrary optimal solution $p^* \in \argmin g$, which is available for free since the $t$th instance is assumed to be solved in \cref{theorem:learning}.

To simply notation, we add elements $\phat_0 = \phat_s = \phat_t = 0$ to $\phat \in \R^V$.
Also, let
$\gamma_{i0} = -\alpha_i$ for $i \in V$,
$\gamma_{0j} = \beta_j$ for $j \in V$,
$\gamma_{sj} = \phat_j$ for $j \in V_0$, and
$\gamma_{it} = -\phat_i$ for $i \in V_0$.
Then, the original edge weights $\wtl_{ij}(\phat)$ can be written as
\begin{align}\label{aeq:wtl}
  \wtl_{ij}(\phat) = \gamma_{ij} - \phat_j + \phat_i \quad \text{for $ij \in \Etl$.}
\end{align}
Since we have $p^* \in \argmin g \subseteq \conv(\argmin g)$, $p^* \in \R^V$ satisfies the inequalities in \eqref{aeq:inequality-system}.
Thus, by additionally defining
$p^*_0 = 0$,
$p^*_s = \max_{j\in V_0} (p^*_j - \phat_j)$, and
$p^*_t = \min_{i \in V_0} (p^*_i - \phat_i)$,
we have
\begin{align}\label{aeq:popt}
  p^*_j - p^*_i \le \gamma_{ij} \quad \text{for $ij \in \Etl$.}
\end{align}
Then, $q \coloneqq p^* - \phat \in \R^{\Vtl}$ is indeed a potential since we have
\begin{align*}
  \wtl_{ij} - q_j + q_i
  = \wtl_{ij} - (p^*_j - \phat_j) +  (p^*_i - \phat_i)
  \overset{\text{\eqref{aeq:wtl}}}{=} \gamma_{ij} - \phat_j + \phat_i - (p^*_j - \phat_j) +  (p^*_i - \phat_i)
  = \gamma_{ij} - p^*_j + p^*_i
  \overset{\text{\eqref{aeq:popt}}}{\ge}
  0
\end{align*}
for all $ij \in \Etl$.
By using this potential, we can obtain non-negative edge weights $\wtl^+_{ij}(\phat)$ as described above.

\section{Missing Proofs in \texorpdfstring{\Cref{sec:minimizer-set}}{Section \ref{sec:minimizer-set}}}
We detail how to obtain an inequality system of $\conv(\argmin g)$ for weighted matroid intersection (\cref{asec:proof-mi-argmin}), discrete energy minimization (\cref{asec:proof-dem-argmin}), and general L-/\Ln-convex minimization under the value-oracle model (\cref{asec:find-ineq-general}).

\subsection{Proof of \texorpdfstring{\Cref{thm:mi-argmin}}{Theorem \ref{thm:mi-argmin}}}\label{asec:proof-mi-argmin}
We discuss the weighted matroid intersection problem, a generalization of various problems such as bipartite matching and packing spanning trees.
A \emph{matroid} $\Mbf$ consists of a finite set $V$ and a non-empty set family $\Bcal \subseteq 2^V$ of \textit{bases} satisfying the following: for any $B_1, B_2 \in \Bcal$ and $i \in B_1 \setminus B_2$, there exists $j \in B_2 \setminus B_1$ such that $B_1 \setminus \set{i} \cup \set{j}$, $B_2 \setminus \set{j} \cup \set{i} \in \Bcal$.
For any $v \in \Z^V$ and $X \subseteq V$, let $v(X) = \sum_{i \in X} v_i$.
For any $v \in \Z^V$ and matroid $\mathbf{M} = (V, \mathcal{B})$, let $\Bcal^v \coloneqq \argmax_{B \in \mathcal{B}} v(B)$ and $\mathbf{M}^v \coloneqq (V, \Bcal^v)$, which also forms a matroid \citep{Edmonds1971-ck}.

Let $\Mbf_1 = (V, \Bcal_1)$ and $\Mbf_2 = (V, \Bcal_2)$ be two matroids on an identical ground set $V$ equipped with weights $w \in \Z^V$.
We assume that the rank of each matroid is at most $r$ (i.e., $\max_{B \in \Bcal_k}{|B|} \le r$ for $k = 1,2$) and that independence oracles of $\Mbf_1$ and $\Mbf_2$ run in $\tau$ time, each of which returns whether input $X \subseteq V$ is a subset of some $B \in \Bcal_k$ or not ($k = 1,2$).
The weighted matroid intersection problem asks to find $B \in \Bcal_1 \cap \Bcal_2$ that maximizes $w(B)$.
Its dual problem is written as minimization of the following L-convex function $g$ (see \citep[Section~3.2]{Sakaue2022-jr}):
\begin{align*}
	\minimize_{p \in \Z^V} \quad
	g(p) = \max_{B \in \Bcal_1} p(B) + \max_{B \in \Bcal_2} (w-p)(B).
\end{align*}
We below prove the following theorem.

\miargmin*

\begin{proof}
  Let $B^* \in \Bcal_1 \cap \Bcal_2$ be any common base that maximizes $w(B^*)$.
  By the strong duality~\citep[Theorem~13.2.4]{Frank2011-rt}, it holds that $w(B^*) = g(p^*)$ for any optimal dual solution $p^* \in \Z^V$.
  This means that $p \in \Z^V$ is optimal if and only if
  \begin{align}\label{eq:mi-complement}
    B^* \in \argmax_{B \in \mathcal{B}_1} p(B) \cap \argmax_{B \in \mathcal{B}_2} (w-p)(B).
  \end{align}
  Furthermore, an optimality condition for linear maximization on matroid bases says that for any $v\in \Z^V$ and matroid $\mathbf{M} = (V, \mathcal{B})$, $B \in \argmax_{B^\prime \in \Bcal} v(B^\prime)$ holds if and only if $v_i \ge v_j$ for all $i \in B$ and $j \in V \setminus B$ with $B \setminus \set{i} \cup \set{j} \in \mathcal{B}$ (see, e.g., \citep[Corollary~39.12b]{Schrijver2003-ol}).
  If we apply this condition to each of $\argmax_{B \in \mathcal{B}_1} p(B)$ and $\argmax_{B \in \mathcal{B}_2} (w-p)(B)$ in \eqref{eq:mi-complement}, the ``if and only if'' condition of \eqref{eq:mi-complement} implies that we can represent $\argmin g$ as follows:
  \begin{align}\label[]{eq:mi-ineq-system}
    \Set*{p \in \Z^V}{\begin{alignedat}{2}
      &\text{$p_i \ge p_j$ for $(i, j) \in E_1(B^*)$}, \\
      &\text{$w_i - p_i \ge w_j - p_j$ for $(i, j) \in E_2(B^*)$} \\
    \end{alignedat}},
  \end{align}
  where $B^*$ is an arbitrary maximum weight common base and $E_k(B^*) = \{\, (i, j) \mathrel{|} i \in B^*, j \in V \setminus B^*, B^* \setminus \set{i} \cup \set{j} \in \mathcal{B}_k\,\}$ for $k = 1, 2$.
  The same inequality system on $\R^V$ represents $\conv(\argmin g)$, as in \cref{proposition:l-convex-set}.

  By the assumption in \cref{thm:mi-argmin}, a maximum weight common base $B^*$ is available.
  Once $B^*$ is given, we can construct the inequality system \eqref{eq:mi-ineq-system} in $\Ord(\tau|B^*||V\setminus B^*|) \lesssim \Ord(\tau nr)$ time, hence $\Tineq = \Ord(\tau nr)$.
\end{proof}

We additionally show that the per-round running time of OGD is $\Ord(\tau nr + n\log n)$.
Since $|E_k| = \Ord(|B^*||V\backslash B^*|) \lesssim \Ord(nr)$ for $k=1,2$, the inequality system \eqref{eq:mi-ineq-system} yields a graph $(\Vtl, \Etl)$ such that $|\Etl| = \Ord(nr)$.
Therefore, Dijkstra's algorithm in the proof of \cref{lem:subgradient} takes $\Ord(nr + n \log n)$ time, and the per-round running time of OGD is $\Tineq + \Ord(nr + n \log n) = \Ord(\tau nr + n \log n)$.
Since solving local local optimization in \cref{alg:steepest-descent} takes $\Tloc = \Ord(\tau nr^{1.5})$ time as in \cref{table:results}, OGD's per-round running time is as short as $\Tloc$ if $\log n \lesssim \Ord(\tau r^{1.5})$.

\subsection{Proof of \texorpdfstring{\Cref{thm:dem-argmin}}{Theorem \ref{thm:dem-argmin}}}\label{asec:proof-dem-argmin}
We discuss discrete energy minimization \citep{Kolmogorov2009-nj}, which appears in computer-vision (CV) applications.
(Although \citet{Kolmogorov2009-nj} considered undirected graphs, a similar result holds for directed graphs as follows; see also \citep[Section~9]{Murota2003-bq}.)
Let $(V, A)$ be a directed graph with $|V| = n$ and $|A| = m$.
Each vertex $i \in V$ and edge $a \in A$ are associated with univariate convex functions $\phi_i:\Z \to \R\cup\set{+\infty}$ and $\psi_a: \Z \to \R\cup\set{+\infty}$, respectively, which we can evaluate in constant time.
Then, a discrete energy minimization problem is written as
\begin{align}\label{eq:prob-discrete-energy-min}
  \minimize_{p \in \Z^V} \;\; g(p) = \sum_{i \in V} \phi_i(p_i) + \sum_{a = ij \in A} \psi_a(p_j - p_i),
\end{align}
which is an \Ln-convex minimization problem and a generalization of minimum-cost flow.
We assume $\dom g \subseteq [0, W]^V$ for some $W>0$; this is usually true in CV applications since the range of pixel values is bounded.
This assumption implies that we can restrict $\dom \phi_i$ and $\dom \psi_a$ to $[0, W] \cap \Z$ and $[-W, +W] \cap \Z$, respectively.

We then introduce the dual problem of \eqref{eq:prob-discrete-energy-min}.
For a vector $\xi \in \R^A$, 
called a \emph{flow}, we define its \emph{boundary} $\partial \xi \in \R^V$ by ${(\partial \xi)}_i = \sum_{a \in \delta^+(i)} \xi_a - \sum_{a \in \delta^-(i)} \xi_a$ for $i \in V$, where $\delta^+(i)$ (resp.\ $\delta^-(i)$) is the set of edges entering to (resp.\ leaving from) $i \in V$.
We can write the dual problem of \eqref{eq:prob-discrete-energy-min} as the following convex cost flow problem:
\[
  \maximize_{\xi \in \R^A}
  \;\;
  f(\xi)
  = \sum_{i \in V} \min_{x \in \Z} \phi_i[-{(\partial \xi)}_i](x)
  + \sum_{a \in A} \min_{y \in \Z} \psi_a[-{\xi}_a](y),
\]
where, for any $u: \R \to \R \cup \set{+\infty}$ and $b \in \R$, $u[b]$ denotes a function defined by $u[b](x) = u(x) + bx$ for $x \in \R$.
The following proposition is useful to characterize the primal-dual structure.
\begin{proposition}[{\citep[Theorem~3.1]{Kolmogorov2009-nj}; cf.\ \citep[Theorem~9.4]{Murota2003-bq}}]\label{thm:energy-argmin}
  Take any flow $\xi^* \in \R^A$ maximizing $f(\xi)$.
  Let $\alpha_i, \beta_i$ be integers with $\argmin \phi_i[-{(\partial \xi^*)}_i] = [\alpha_i, \beta_i] \cap \Z$ for $i \in V$, and let $\check{\gamma}_a, \hat{\gamma}_a$ be integers with $\argmin \psi_a[-{\xi^*}_a] = [\check{\gamma}_a, \hat{\gamma}_a] \cap \Z$ for $a \in A$.
  Then, we can represent $\argmin g$ as follows:
  \begin{align}\label{eq:discrete-energy-min-ineq-system}
    \Set*{p^* \in \Z^V}{\begin{alignedat}{2}
      &\text{$\alpha_i \le p^*_i \le \beta_i$ for $i \in V$}, \\
      &\text{$\check{\gamma}_a \le p^*_j - p^*_i \le \hat{\gamma}_a$ for $a = ij \in A$} \\
    \end{alignedat}}.
  \end{align}
\end{proposition}

We now prove the following theorem.

\demargmin*

\begin{proof}
  From \cref{thm:energy-argmin}, given an arbitrary optimal flow $\xi^*$, we can represent $\argmin g$ by the inequality system~\eqref{eq:discrete-energy-min-ineq-system}.
  We can find one such $\xi^*$ by using an algorithm of \citep{Ahuja2003-qz} in $\Ord(mn\log(n^2/m)\log(nW))$ time.
  Then, since we can restrict $\dom \phi_i$ and $\dom \psi_a$ to $[0, W] \cap \Z$ and $[-W, +W] \cap \Z$, respectively, we can locate values of $\alpha_i, \beta_i, \check{\gamma}_a, \hat{\gamma}_a \in \Z$ for all $i \in V$ and $a \in A$ in $\Ord((n + m)\log W)$ time via binary search (which is faster than the algorithm for computing $\xi^*$).
  Therefore, we can obtain the inequality system \eqref{eq:discrete-energy-min-ineq-system} in $\Tineq = \Ord(mn\log(n^2/m)\log(nW))$ time, and the same system on $\R^V$ represents $\conv(\argmin g)$.
\end{proof}

\subsection{Proof of \texorpdfstring{\Cref{thm:gen-argmin}}{Theorem \ref{thm:gen-argmin}}}\label{asec:find-ineq-general}
We consider L-/\Ln-convex function minimization, $\min_{p \in \Z^V} g(p)$, with a value oracle of $g$.
We prove the following theorem.

\genargmin*

\begin{proof}
Recall that prediction $\phat \in \R^V$, at which we compute a subgradient of $\mubar(\cdot; g)$, is always contained in $[-C, +C]^V$ since OGD performs the $\ell_2$-projection onto $[-C, +C]^V$.
First, we show that an inequality system of $S = \conv(\argmin g) \cap [-2C, +2C]^V$ is sufficient for computing a subgradient; that is,  we only need to obtain an inequality system of the form
\begin{align*}
  \text{$\alpha_i \le p_i \le \beta_i$ for $i \in V$}
  \quad
  \text{and}
  \quad
  \text{$p_j - p_i \le \gamma_{ij}$ for distinct $i, j \in V$}
\end{align*}
such that
\begin{align}\label{eq:ranges}
  \alpha_{i} \in [-2C, +2C] \cap \Z,
  &&
  \beta_i \in [-2C, +2C] \cap \Z,
  &&
  \text{and}
  &&
  \gamma_{ij} \in [-4C, +4C] \cap \Z.
\end{align}
As discussed in \cref{subsec:subgradient}, we can compute a subgradient of $\mubar(\phat; g)$ by finding a shortest $s$--$t$ path in $(\Vtl, \Etl)$ with weights $\wtl_{ij}(\phat)$.
This procedure is indeed equivalent to computing an $\ell^\pm_\infty$-projection of $\phat$ onto the convex hull of an \Ln-convex set, $\conv(\argmin g)$ (see \citep[Appendix D]{Sakaue2022-jr}).
Since we have $\phat \in  [-C, +C]^V$ and $\argmin g \cap [-C, +C]^V \neq \emptyset$, such an $\ell^\pm_\infty$-projection of $\phat$ never goes out of $[-2C, +2C]^V$.
Therefore, among all inequalities representing $\conv(\argmin g)$, those that do not intersect with $[-2C, +2C]^V$ can be ignored when computing a shortest path in $(\Vtl, \Etl)$.
This implies that we can compute a subgradient of $\mubar(\phat; g)$ if we have an inequality system of $S = \conv(\argmin g) \cap [-2C, +2C]^V$.

We then describe how to obtain an inequality system of $S$.
From the above discussion, we can focus on searching for appropriate values of $\alpha_i$, $\beta_i$, and $\gamma_{ij}$ satisfying \eqref{eq:ranges}.
We seek such values via binary search as follows.
Consider, for example, doing binary search on $[-4C, +4C]$ to find an appropriate $\gamma_{ij}$ value.
Given a current $\gamma_{ij}$ value, the inequality $p_j - p_i \le \gamma_{ij}$ is valid for all minimizers $p \in S$ if and only if adding a constraint $p_j - p_i \ge \gamma_{ij} + 1$ to $\min\Set*{g(p)}{{p \in [-2C, +2C]^V \cap \Z^V}}$ increases the minimum value.
We can check whether the latter is true by solving the \Ln-convex minimization problem with an effective domain restricted to $[-2C, +2C]^V \cap \Set*{p \in \Z^V}{p_j - p_i \ge \gamma_{ij} + 1}$.
By using the steepest descent scaling algorithm \citep[Section 10.3.2]{Murota2003-bq}, we can solve the problem via $\Ord(\log C)$ times submodular function minimization, each of which can be solved with an $\Ord(\mathrm{EO} \cdot n^3\log^2n + n^4\log^{\Ord(1)}n)$-time algorithm of \citep{Lee2015-ia}.
By repeating this test $\Ord(\log C)$ times, the binary search produces a tight inequality $p_j - p_i \le \gamma_{ij}$ (i.e., it defines a facet of $\conv(S)$).
Similarly, we can find $\alpha_i$ and $\beta_i$ values.
To find all $\alpha_i$, $\beta_i$, and $\gamma_{ij}$ values, we perform binary search $\Ord(n^2)$ times.
Therefore, the total computation time for obtaining the desired inequality system is $\Tineq = \Ord(n^2\log^2 C \cdot (\mathrm{EO}\cdot n^3\log^2n + n^4\log^{\Ord(1)}n))$.
\end{proof}

\section{Experiments}\label{asec:experiment}
We conducted experiments with small random bipartite matching instances generated as follows.
Let $n=10$ and $(L, R)$ be a bipartition of a vertex set $V$ such that $L = \set{1,2,\dots,5}$ and $R=\set{6,7,\dots,10}$.
First, we created edges $(i,i+|L|) \in L \times R$ for $i = 1,\dots,5$ with a weight of $1$ to ensure that there always exists at least one perfect matching.
Then, for the other pairs of $(i, j) \in L \times R$, we let $w_{ij} = i\times (j - |L|) + u$, where $u$ is a noise term drawn uniformly at random from $[-\sigma, +\sigma]\cap\Z$ for some $\sigma > 0$.
If $w_{ij} > 0$, we created edges $(i,j) \in L\times R$ with weights $w_{ij}$.
Hence, predictions should be learned to match $i \in L$ and $j \in R$ such that both $i$ and $j$ are large, while $i \in L$  and $i + |L| \in R$ should not be matched since $(i, i + |L|)$ only has an edge weight of $1$.
We thus created a dataset of $T = 1000$ such random weighted bipartite graphs.
We repeated this procedure $10$ times to obtain $10$ independent datasets, with which we calculated the mean and standard deviation of the results.
We made such datasets for various noise strengths $\sigma \in \set{1,5,10,20}$.

We consider the online learning setting where the random weighted bipartite graphs arrive sequentially for $t = 1,\dots,T$.
Each method learns predictions $\phat_t \in \R^V$ for $t = 1,\dots,T$ using OGD, where the loss function $f_t$ differs among the methods, as described shortly.
To improve the empirical performance, we used the following refined variant of OGD:
we used the anytime online-to-batch scheme \citep{Cutkosky2019-ik} for the last iterate convergence of each $t$th prediction (see \cref{asec:proof-of-cor-sample} for details) and an adaptive learning rate of $\eta_t = \frac{C\sqrt{2n}}{\sqrt{\sum_{t'=1}^t \norm{z_{t'}}_2^2}}$ \cite{Streeter2010-bn} in each $t$th iteration of OGD, where $z_t$ is the $t$th subgradient, $C = nW$, and $W$ is the largest edge weight in a dataset (as in \cref{rem:C}).
Furthermore, since OGD's performance was sensitive to the scale of learning rates, we used rescaled learning rates $\rho \times \eta_t$ for $\rho \in \set{0.01, 0.1, 1.0, 10.0}$.
We compared methods with three types of loss functions:
$f_t(\phat) = \norm{p^*_t - \phat}_1$ \citep{Dinitz2021-sd},
$f_t(\phat) = \norm{p^*_t - \phat}_\infty$ \citep{Sakaue2022-jr},
and
$f_t(\phat) = \mubar(\phat; g_t)$ (ours), where the first two used $p^*_t$ returned by \cref{alg:steepest-descent} (or the Hungarian method) warm-started by $\phat_t$.
We also used the cold-start method as a baseline, which always set $\phat_t = \zeros$.
We denote those methods by $\ell_1$, $\ell_\infty$, $\mubar$, and Cold, respectively, for short.
To convert prediction $\phat_t$ into an initial feasible solution, $\ell_1$ used the greedy algorithm in \citep{Dinitz2021-sd}, while $\ell_\infty$, $\mubar$, and Cold used the $\ell^\pm_\infty$-projection and rounding, as in \cref{theorem:dca-framework} (a specific procedure for bipartite matching is presented in \citep[Section 3.1]{Sakaue2022-jr}).

\begin{figure}[tbp]
  \centering
  \includegraphics[width=.79\columnwidth]{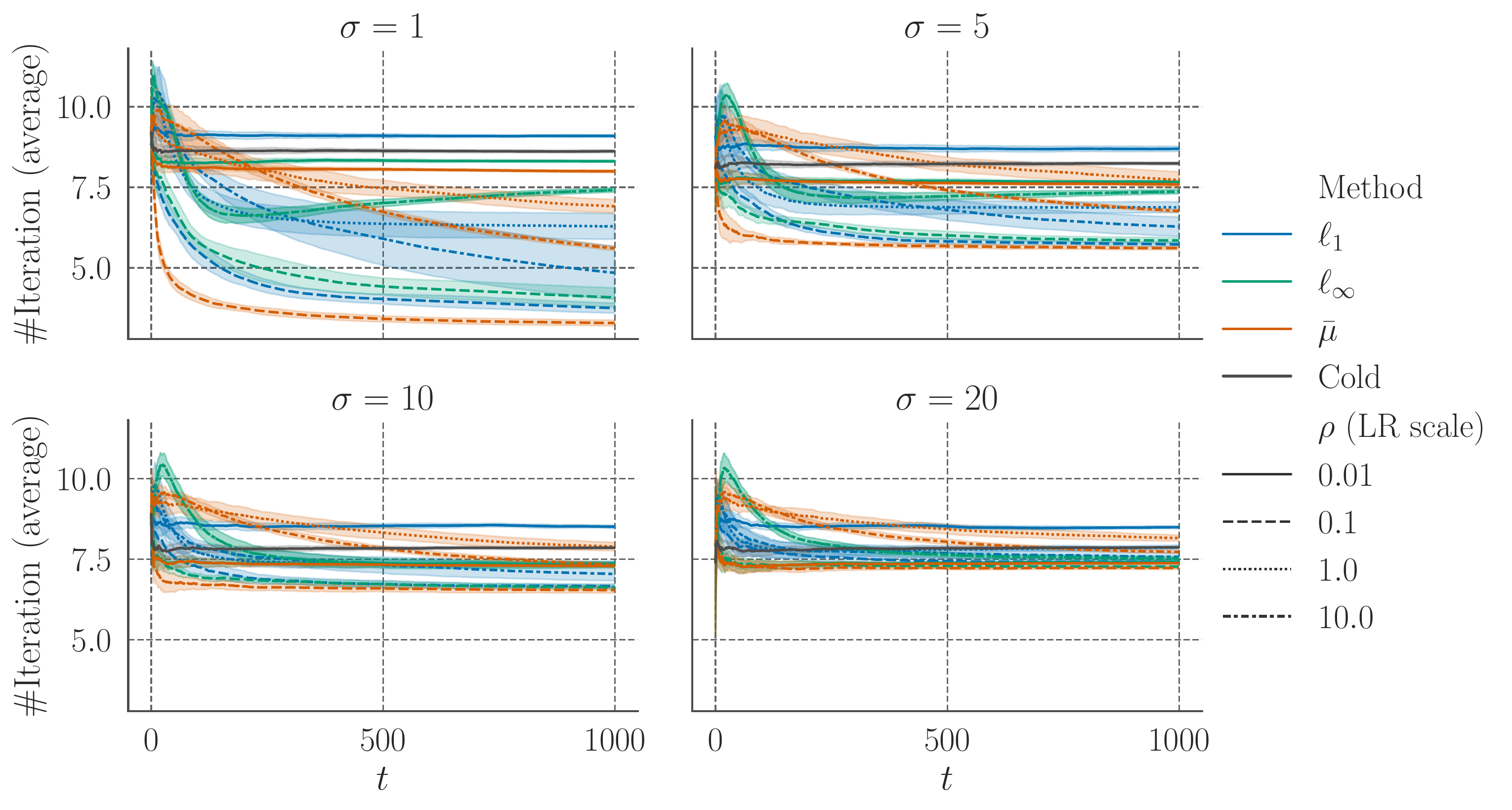}
  \caption{The average number of iterations of \cref{alg:steepest-descent} for bipartite matching warm-started with predictions learned by each method (the average is taken over the past $t$ instances).
  Here, $\sigma$ represents the noise strength and $\rho$ is the scaling factor of the learning rate (LR) of OGD. The error band indicates the standard deviation over $10$ independently random datasets.}
  \label{fig:result}
\end{figure}

\Cref{fig:result} shows the average number of iterations of \cref{alg:steepest-descent} (or the Hungarian method) warm-started with predictions $\phat_t$ learned by each method, where the average is taken over the past $t$ instances for $t=1,\dots,T$.
In the low-noise setting ($\sigma = 1$), $\ell_1$, $\ell_\infty$, and $\mubar$ with $\rho = 0.1$ significantly outperformed Cold, while their advantages became smaller as the noise strength $\sigma$ increased, as is also observed in \citep[Section~4]{Dinitz2021-sd}.
In every case, our $\mubar$ with $\rho = 0.1$ returned the best predictions, leading to the smallest number of iterations.
Moreover, $\mubar$ with $\rho = 0.1$ decreased the number of iterations more quickly than the other methods, implying that it can learn good predictions from a smaller number of sampled instances.

\end{document}